\documentclass{article}

\usepackage[nonatbib, preprint]{neurips_2024}

\usepackage[utf8]{inputenc} 
\usepackage[T1]{fontenc}    
\usepackage{hyperref}       
\usepackage{url}            
\usepackage{booktabs}       
\usepackage{amsfonts}       
\usepackage{nicefrac}       
\usepackage{microtype}      
\usepackage{xcolor}         
\usepackage{amsmath,amsthm,amssymb,mathtools}
\usepackage{adjustbox,makecell,booktabs}
\usepackage{float}
\newtheorem{definition}{Definition}[section]
\newtheorem{lemma}{Lemma}[section]
\newtheorem{theorem}{Theorem}[section]
\newtheorem{proposition}{Proposition}[section]
\newtheorem{corollary}{Corollary}[section]
\newtheorem{remark}{Remark}[section]

\title{Integral Signatures of Activation Functions: A 9-Dimensional Taxonomy and Stability Theory for Deep Learning}

\author{%
  Ankur Mali \\
  Bellini College of AI, \\
  Cybersecurity and Computing \\
  University of South Florida \\
  Tampa, FL 33617 \\
  \texttt{ankurarjunmali@usf.edu} \\
  \And
  Lawrence Hall \\
  Bellini College of AI,\\
  Cybersecurity and Computing\\
  University of South Florida\\
  Tampa, FL 33620 \\
  \texttt{lohall@usf.edu} \\
  \AND
  Jake Williams\\
  College of Computing and Informatics\\
  Drexel University\\
  Philadelphia, PA 19104 \\
  \texttt{jw3477@drexel.edu} \\
  \And
  Gordon Richards  \\
  Departments of Physics\\
  Drexel University\\
  Philadelphia, PA 19104 \\
  \texttt{gtr25@drexel.edu} \\
}

\begin{document}

\maketitle

\begin{abstract}
In this work, we propose a principled, integral framework for classifying activation functions via a nine-dimensional signature
\[
\mathcal{S}_\sigma(\phi)=\big(m_1,g_1,g_2,m_2,\eta,\alpha_+,\alpha_-,\mathrm{TV}(\phi'),C(\phi)\big),
\]
which couples Gaussian propagation statistics $(m_1,g_1,g_2,m_2,\eta)$ with asymptotic geometry $(\alpha_+,\alpha_-)$ and regularity measures $(\mathrm{TV}(\phi'),C(\phi))$.  We establish well-posedness, prove affine reparameterization laws including bias handling, and provide a closure theorem under bounded slope variation via Helly's selection. On the dynamical side, we develop contraction-based Lyapunov theorems with explicit descent constants for scalar recursions $T(x)=\phi(ax+b)$, complemented by primitive-based Lyapunov functions using $F=\int \phi$. For wide neural networks, we express variance propagation as $q_{\ell+1}=\sigma_W^2\,m_2(\sqrt{q_\ell})+\sigma_b^2$ and characterize the stability region through signature components $(m_2',g_2)$. From a kernel perspective, we derive dimension-free mixed-Hessian bounds $\|\nabla_x\nabla_y K(x,y)\|_{\mathrm{op}} \leq C g_2(\|x\|)g_2(\|y\|)\|x\|\|y\|$ with universal constant $C \leq 4$, and establish a bounded variation corollary relating kernel smoothness to $\mathrm{TV}(\phi')$. We systematically classify eight standard activations (ReLU, leaky-ReLU, $\tanh$, sigmoid, Swish, GELU, Mish, TeLU), proving finite or infinite behavior for all signature components and revealing fundamental distinctions between saturating, linear-growth, and smooth activation families. Numerical evaluation via Gauss-Hermite quadrature validates our theoretical predictions, with Monte Carlo cross-checks confirming 6+ digit accuracy for the Gaussian expectation components $(m_1,g_1,g_2,m_2,\eta)$ across multiple input scales. The framework provides actionable design principles for activation selection, moving beyond empirical comparisons toward theoretically grounded choices that optimize stability, bias control, and kernel conditioning. This integral signature approach establishes a rigorous mathematical foundation for activation function analysis, enabling systematic design guided by provable dynamical properties rather than trial-and-error experimentation.
\end{abstract}

\section{Introduction}\label{sec:intro}

Activation functions sit at the heart of modern neural networks. They are not only the
mechanism that injects nonlinearity into deep models but also a key determinant of their
\emph{expressive power}, \emph{stability}, and \emph{learnability}
\cite{cybenko1989approximation,hornik1991approximation, stogin2024provably, omlin1996constructing, mitarchuk2025saturation, fernandez2025teluactivationfunctionfast}.
The long evolution of nonlinearities---from sigmoidal units in early perceptrons
\cite{rosenblatt1958perceptron}, through $\tanh$ and logistic units in recurrent models
\cite{elman1990finding}, to the rectified linear unit (ReLU) \cite{Fukushima1969VisualFE} and its many
variants such as leaky-ReLU \cite{maas2013rectifier}, Swish \cite{ramachandran2017searching}, GELU \cite{hendrycks2016gaussian}, Mish \cite{misra2019mish}, and more recently TeLU
\cite{fernandez2025teluactivationfunctionfast}---has
enabled successive breakthroughs in optimization, generalization, and robustness across
vision, language, and scientific domains. 
Yet, despite this proliferation, comparative analyses remain dominated by
\emph{heuristic, dataset-specific benchmarks} rather than a principled mathematical
framework.

\vspace{0.5em}
\noindent\textbf{From heuristics to analytic structure.}
Prior theoretical attempts have captured fragments of activation behavior.
Initialization schemes such as Xavier and He normalization
\cite{glorot2010understanding,he2015delving} quantify second moments to stabilize early
training. Mean-field and dynamical isometry analyses
\cite{poole2016exponential,lee2017deep,pennington2017resurrecting,hanin2018start}
study Gaussian statistics of $\phi$ and $\phi'$ to control signal propagation.  
Kernel-based perspectives, including the Neural Tangent Kernel (NTK) and Gaussian Process
limits \cite{jacot2018neural,lee2017deep}, connect activations to smoothness of induced
reproducing kernels.  
In parallel, approximation theory has revealed how the smoothness of the target function
and the activation interact to determine expressive power
\cite{ohn2019smooth,langer2020approximation,boulle2020rational}.  
Recent proposals such as the Smooth Maximum Unit (SMU) \cite{biswas2022smu} emphasize
that smooth activations can improve gradient stability while retaining ReLU-like
behavior.  
While powerful, these frameworks isolate a few statistics (variance, slope, Lipschitz
constants), leaving open the challenge of a \emph{unified taxonomy} that integrates
Gaussian propagation, asymptotics, and variational structure in a single analytic object.

\vspace{0.5em}
\noindent\textbf{Integral signatures.}
We propose such a unification via \emph{integral signatures}.
The central observation is that forward and backward dynamics in wide layers are governed
by Gaussian integrals of $\phi$ and its derivative, while tail slopes and curvature
govern drift and kernel regularity.  
Accordingly, for $Z\sim\mathcal{N}(0,\sigma^2)$ we define Gaussian moments
\[
m_k(\sigma) \;=\; \int_{\mathbb{R}} \phi(x)^k \,\frac{1}{\sqrt{2\pi}\sigma}
e^{-x^2/(2\sigma^2)}\,dx \;=\; \mathbb{E}[\phi(Z)^k],
\]
and pair them with asymptotic slopes, total variation of $\phi'$ \footnote{$\phi$ generally refers to any activation under consideration}, and a
tail-compensated curvature primitive. These quantities generalize the statistics used in
initialization, mean-field criticality, and NTK smoothness into a single finite
signature.

\vspace{0.5em}
\noindent\textbf{A nine-dimensional taxonomy.}
We introduce a nine-dimensional integral signature
\[
\mathcal{S}_\sigma(\phi)
=\big(m_1,\,g_1,\,g_2,\,m_2,\,\eta,\,\alpha_+,\,\alpha_-,\,\mathrm{TV}(\phi'),\,C(\phi)\big),
\]
where $(m_1,g_1,g_2,m_2,\eta)$ are Gaussian expectations of $\phi$ and $\phi'$;
$(\alpha_+,\alpha_-)$ encode directional asymptotes; $\mathrm{TV}(\phi')$ measures slope
variation; and $C(\phi)$ quantifies compensated curvature.  
Unlike prior approaches that emphasize one aspect (variance stability, isometry,
or kernel limits), $\mathcal{S}_\sigma(\phi)$ is:
\emph{affine-aware}, \emph{closed under limits}, and \emph{predictive} for
propagation stability, Lyapunov descent, and kernel regularity.  
It thus extends and unifies earlier “taxonomies” into a mathematically robust,
finite-dimensional framework.

\vspace{0.5em}
\noindent\textbf{Contributions.}
This work develops the first comprehensive integral taxonomy of activation functions.
Our main contributions are:
\begin{enumerate}
    \item \textbf{Integral taxonomy:} Formalization of $\mathcal{S}_\sigma(\phi)$ with
    proofs of well-posedness under mild analytic conditions.
    \item \textbf{Affine reparameterization:} Exact transformation laws under scaling and
    bias shifts $\tilde\phi(x)=c\,\phi(ax+b)+d$.
    \item \textbf{Closure under limits:} Compactness and convergence of signatures under
    bounded slope variation.
    \item \textbf{Stability theorems:} Lyapunov-based guarantees for affine maps of $\phi$ and
    mean-field variance recursions characterizing criticality.
    \item \textbf{Kernel regularity:} Dimension-free operator norm bounds for kernels induced by $\phi$,
    with BV-based refinements.
    \item \textbf{Classification:} Rigorous placement of ReLU, leaky-ReLU, $\tanh$,
    sigmoid, Swish, GELU, Mish, and TeLU in the taxonomy.
    \item \textbf{Numerical evaluation:} Gauss–Hermite quadrature and Monte Carlo validation across scales $\sigma$.
    \item \textbf{Design guidance:} Actionable principles linking signature coordinates to initialization,
    drift control, and kernel smoothness.
\end{enumerate}

\vspace{0.5em}
\noindent\textbf{Paper organization.}
Section~\ref{sec:prelim} sets notation and background. 
Section~\ref{sec:taxonomy} develops the integral taxonomy and structural results.
Section~\ref{sec:classification} applies the taxonomy to standard activations.
Section~\ref{sec:Lyapunov} presents stability and kernel analyses. 
Section~\ref{sec:Numerical} provides numerical evaluation. 
Related work is in Section~\ref{sec:related}; we conclude in
Section~\ref{sec:conclusion}.

\section{Preliminaries}\label{sec:prelim}

We collect the measure-theoretic and probabilistic notions used throughout and fix notation.

\subsection{Measures and \texorpdfstring{$L^p$}{Lp} Spaces}

\begin{definition}[Lebesgue $L^p$ spaces]
For $1\le p<\infty$,
\[
L^p(\mathbb{R}) \;:=\; \Big\{ f:\mathbb{R}\to\mathbb{R}\ \Big|\ \int_{\mathbb{R}} |f(x)|^p\,dx < \infty \Big\},
\]
and $L^\infty(\mathbb{R})$ is the space of essentially bounded measurable functions.
\end{definition}

\begin{definition}[Gaussian measure and $L^p(\gamma_\sigma)$]
For $\sigma>0$ let
\[
d\gamma_\sigma(x) \;:=\; \frac{1}{\sqrt{2\pi}\,\sigma}\,e^{-x^2/(2\sigma^2)}\,dx,
\qquad Z\sim\mathcal{N}(0,\sigma^2).
\]
Then
\[
L^p(\mathbb{R},\gamma_\sigma) \;:=\; \Big\{ f:\mathbb{R}\to\mathbb{R}\ \Big|\ \int_{\mathbb{R}} |f(x)|^p\,d\gamma_\sigma(x) < \infty \Big\}.
\]
\end{definition}

\begin{remark}[Lebesgue vs.\ Gaussian integrability]\label{rem:lp-correct}
On $(\mathbb{R},dx)$ (infinite measure), any non-decaying function (e.g., bounded with nonzero limits such as $\tanh$, $\operatorname{sigm}$) is \emph{not} in $L^p(\mathbb{R})$ for any finite $p$. Likewise, any function with polynomial \emph{growth} (e.g., $\mathrm{ReLU}$, $x^k$) is not in $L^p(\mathbb{R})$ for finite $p$. In contrast, under Gaussian measure, polynomial growth is admissible: if $|\phi(x)|\le C(1+|x|^r)$, then $\phi\in L^p(\mathbb{R},\gamma_\sigma)$ for all $p\in[1,\infty)$ and all $\sigma>0$.
\end{remark}

\subsection{Gaussian Moment and Derivative Signatures}

\begin{definition}[Gaussian moment signatures]
For a measurable activation $\phi:\mathbb{R}\to\mathbb{R}$ and $k\in\mathbb{N}$,
\[
m_k(\sigma) \;:=\; \int_{\mathbb{R}} \phi(x)^k\,d\gamma_\sigma(x)
\;=\; \mathbb{E}\big[\phi(Z)^k\big],\qquad Z\sim\mathcal{N}(0,\sigma^2).
\]
We write $m_1,m_2$ for mean and second moment under $\gamma_\sigma$.
\end{definition}

\begin{definition}[Derivative and mixed signatures]
Assume $\phi$ is a.e.\ differentiable with at most polynomial growth. Define
\[
g_k(\sigma):=\Big(\mathbb{E}[\phi'(Z)^k]\Big)^{1/k},\qquad
\eta(\sigma):=\mathbb{E}[Z\,\phi(Z)].
\]
\end{definition}

\begin{lemma}[Gaussian integration by parts and differentiation]\label{lem:gauss-ibp}
Let $Z=\sigma G$ with $G\sim\mathcal{N}(0,1)$. If $\phi$ has at most polynomial growth and $\phi'$ is locally integrable, then
\[
\eta(\sigma)\;=\;\mathbb{E}[Z\,\phi(Z)] \;=\; \sigma^2\,\mathbb{E}[\phi'(Z)] \;=\; \sigma^2\,g_1(\sigma).
\]
Moreover, $m_2(\sigma)=\mathbb{E}[\phi(\sigma G)^2]$ is differentiable with
\[
m_2'(\sigma) \;=\; \mathbb{E}\big[2\,\phi(\sigma G)\,\phi'(\sigma G)\,G\big]
\;=\; \frac{1}{\sigma}\,\mathbb{E}\big[2\,\phi(Z)\,\phi'(Z)\,Z\big].
\]
\end{lemma}

\begin{proof}
The first identity is the standard Gaussian integration-by-parts:
$\mathbb{E}[G f(G)] = \mathbb{E}[f'(G)]$ for suitable $f$; apply with $f(g)=\phi(\sigma g)$.
Differentiate $m_2(\sigma)=\mathbb{E}[\phi(\sigma G)^2]$ under the expectation and use the chain rule.
\end{proof}

\begin{remark}[On the role of moment sequences]
We use $\{m_k(\sigma)\}$ (and $g_1,g_2,\eta$) as a \emph{signature} to control propagation and stability. We do not rely on distributional uniqueness of $\phi(Z)$ from moments (e.g., Carleman determinacy is not assumed).
\end{remark}

\subsection{Asymptotic Behavior: Growth and Slopes}

\begin{definition}[Polynomial growth exponent]
For measurable $\phi:\mathbb{R}\to\mathbb{R}$,
\[
\alpha(\phi)\;:=\;\inf\Big\{\alpha\ge 0\ \Big|\ \limsup_{|x|\to\infty}\frac{|\phi(x)|}{|x|^\alpha}<\infty\Big\}.
\]
Examples: $\mathrm{ReLU}(x)$ has $\alpha(\phi)=1$; $\tanh(x)$ and $\operatorname{sigm}(x)$ have $\alpha(\phi)=0$; $\phi(x)=x^2$ has $\alpha(\phi)=2$.
\end{definition}

\begin{definition}[Asymptotic linear slopes]
When $\alpha(\phi)\le 1$, we refine tails via one-sided linear slopes:
\[
\alpha_+\;:=\;\lim_{x\to+\infty}\frac{\phi(x)}{x},\qquad
\alpha_-\;:=\;\lim_{x\to-\infty}\frac{\phi(x)}{x},
\]
whenever the limits exist (finite). If limits do not exist, we use
\[
\overline{\alpha}_\pm := \limsup_{x\to\pm\infty}\frac{\phi(x)}{x},\qquad
\underline{\alpha}_\pm := \liminf_{x\to\pm\infty}\frac{\phi(x)}{x}.
\]
\end{definition}

\begin{remark}[Integrability implications]
If $\limsup_{|x|\to\infty}|\phi(x)|>0$ (bounded nonzero tails) or $|\phi(x)|\gtrsim |x|^\beta$ eventually for some $\beta\ge 0$, then $\phi\notin L^p(\mathbb{R})$ for any finite $p$.
Under Gaussian measure, any polynomial growth (finite $\alpha(\phi)$) implies $\phi\in L^p(\gamma_\sigma)$ for all $p<\infty$ and all $\sigma>0$.
\end{remark}

\paragraph{Notation.}
We reserve $\sigma>0$ exclusively for \emph{Gaussian standard deviation}, and we denote the logistic activation by
\[
\operatorname{sigm}(x):=(1+e^{-x})^{-1}.
\]
Expectations $\mathbb{E}[\cdot]$ are taken with respect to the indicated measure (Lebesgue or $\gamma_\sigma$) by context.

\paragraph{Standing assumptions.}
Unless stated otherwise, activations $\phi$ have at most polynomial growth and are a.e.\ differentiable with locally integrable $\phi'$, ensuring all Gaussian functionals above are well defined and allowing differentiation under the integral sign.

\subsection{Moment Finiteness: Sufficient Growth Conditions}\label{subsec:moment-finiteness}

We give simple, checkable conditions ensuring $g_2(\sigma)$ and $g_4(\sigma)$ are finite for all $\sigma>0$,
together with explicit bounds useful in kernel curvature estimates.

\begin{lemma}[Polynomial growth of the slope suffices]\label{lem:g2g4-growth}
Let $\phi$ be a.e.\ differentiable, and suppose there exist constants $A,B\ge 0$ and $r\ge 0$ such that
\[
|\phi'(x)| \;\le\; A + B\,|x|^r \qquad \text{for all }x\in\mathbb{R}.
\]
Then for every $\sigma>0$,
\[
g_2(\sigma)\;=\;\Big(\mathbb{E}[\phi'(Z)^2]\Big)^{1/2} \;<\; \infty,
\qquad
g_4(\sigma)\;=\;\Big(\mathbb{E}[\phi'(Z)^4]\Big)^{1/4} \;<\; \infty,
\]
with explicit bounds
\[
g_2(\sigma)
\;\le\;
\sqrt{2}\,\Big(A^2 + B^2\,\sigma^{2r}\,M_{2r}\Big)^{1/2},\qquad
g_4(\sigma)
\;\le\;
2^{3/4}\,\Big(A^4 + B^4\,\sigma^{4r}\,M_{4r}\Big)^{1/4},
\]
where $Z\sim\mathcal{N}(0,\sigma^2)$ and $M_k:=\mathbb{E}|G|^k$ with $G\sim\mathcal{N}(0,1)$ given by
\[
M_k \;=\; \mathbb{E}|G|^k \;=\; \frac{2^{k/2}}{\sqrt{\pi}}\;\Gamma\!\Big(\frac{k+1}{2}\Big)\qquad (k>-1).
\]
\end{lemma}

\begin{proof}
For $p\in\{2,4\}$, by $(a+b)^p\le 2^{p-1}(a^p+b^p)$ for $a,b\ge 0$,
\[
|\phi'(Z)|^p \;\le\; 2^{p-1}\big(A^p + B^p |Z|^{pr}\big).
\]
Taking expectations with $Z=\sigma G$ gives
$\mathbb{E}|Z|^{pr}=\sigma^{pr}\mathbb{E}|G|^{pr}=\sigma^{pr}M_{pr}$, hence the displayed bounds.
\end{proof}

\begin{corollary}[Uniform Lipschitz or BV slope]\label{cor:g2g4-bv}
If $\sup_x|\phi'(x)|\le M<\infty$ (i.e., $\phi$ is globally $M$-Lipschitz), then
\[
g_2(\sigma)\le M,\qquad g_4(\sigma)\le M\qquad\text{for all }\sigma>0.
\]
In particular, if $\phi'\in BV(\mathbb{R})$ with one-sided limits, then
\[
\sup_{x\in\mathbb{R}}|\phi'(x)| \;\le\; \mathrm{TV}(\phi') + c_\star,
\qquad c_\star:=|\phi'(+\infty)-\phi'(-\infty)|,
\]
so $g_2(\sigma),g_4(\sigma)\le \mathrm{TV}(\phi')+c_\star$ uniformly in $\sigma$.
\end{corollary}

\begin{remark}[Justifying differentiation under the integral]
Under the growth condition of Lemma~\ref{lem:g2g4-growth}, both $\phi(\sigma G)$ and $\phi'(\sigma G)$
have finite moments of all orders. Therefore dominated convergence applies to the Gaussian expectations
in Lemma~\ref{lem:gauss-ibp}, legitimizing differentiation under the expectation for $m_2(\sigma)$ and related functionals.
\end{remark}

\begin{remark}[Typical cases]
\begin{itemize}
\item \emph{ReLU / leaky-ReLU:} $\phi'$ is bounded ($M\le 1$), hence $g_2(\sigma),g_4(\sigma)\le 1$.
\item \emph{GELU, Swish, Mish, TeLU:} $\phi'$ has at most polynomial growth (indeed, bounded for GELU/Swish/Mish/TeLU), so Lemma~\ref{lem:g2g4-growth} applies.
\item \emph{Polynomial activations:} $\phi'(x)=O(|x|^{r})$ for some integer $r\ge 0$, giving the explicit polynomial-in-$\sigma$ bounds above.
\end{itemize}
\end{remark}

\section{Integral Taxonomy and the 9-Dimensional Signature}\label{sec:taxonomy}

We formalize the taxonomy via a nine-dimensional \emph{integral signature} that encodes propagation statistics, tail asymptotics, and slope regularity. Throughout, fix $\sigma>0$ and let $Z\sim\mathcal{N}(0,\sigma^2)$.

\subsection{Definitions}

\begin{definition}[Primitive and tail-compensated primitive]
Let $\phi:\mathbb{R}\to\mathbb{R}$ be locally absolutely continuous and set
\[
F(x):=\int_0^x \phi(t)\,dt.
\]
If the one-sided linear slopes
\[
\alpha_+ := \lim_{x\to+\infty}\frac{\phi(x)}{x},\qquad
\alpha_- := \lim_{x\to-\infty}\frac{\phi(x)}{x}
\]
exist as finite real numbers, define the tail-compensated primitive
\[
F_{\mathrm{asym}}(x)
:= \int_0^x \Big(\phi(t) - \alpha_+\, t\,\mathbf{1}_{\{t\ge 0\}}
                           - \alpha_-\, t\,\mathbf{1}_{\{t< 0\}}\Big)\,dt.
\]
We write
\[
C(\phi):=\sup_{x\in\mathbb{R}} |F_{\mathrm{asym}}(x)|\in[0,\infty]
\]
whenever $F_{\mathrm{asym}}$ is well defined (see Lemma~\ref{lem:fasym-finiteness} for sufficient conditions for finiteness).
\end{definition}

\begin{definition}[Gaussian statistics]
For $Z\sim\mathcal{N}(0,\sigma^2)$, set
\[
m_k(\sigma):=\mathbb{E}\big[\phi(Z)^k\big],\qquad
g_1(\sigma):=\mathbb{E}\big[\phi'(Z)\big],\qquad
g_2(\sigma):=\Big(\mathbb{E}\big[\phi'(Z)^2\big]\Big)^{1/2},\qquad
\eta(\sigma):=\mathbb{E}\big[\phi(Z)\,Z\big].
\]
Here $\phi'$ is the a.e.\ derivative (for kinked activations this is understood a.e.\ with respect to Lebesgue measure).
\end{definition}

\begin{definition}[Slope variation]
If $\phi'$ is of bounded variation on $\mathbb{R}$ (i.e., $\phi'\in BV(\mathbb{R})$), its \emph{total variation}
is
\[
\mathrm{TV}(\phi') \;:=\; \|\phi''\|_{\mathcal{M}}(\mathbb{R}),
\]
the total variation of the (distributional) second-derivative measure $\phi''$.
For piecewise $C^1$ activations with finitely many kinks, $\mathrm{TV}(\phi')$ equals the sum of
$\int |\phi''|$ on smooth pieces plus the magnitudes of slope jumps at the kinks.
\end{definition}

\begin{definition}[Nine-dimensional signature]
The \emph{integral signature at scale $\sigma$} is
\[
\boxed{\;
\mathcal{S}_\sigma(\phi)
:= \Big(m_1(\sigma),\, g_1(\sigma),\, g_2(\sigma),\, m_2(\sigma),\, \eta(\sigma),\,
\alpha_+,\, \alpha_-,\, \mathrm{TV}(\phi'),\, C(\phi)\Big).
\;}
\]
\end{definition}

\begin{remark}[Interpretation for network dynamics]
$m_2$ governs mean-field variance propagation; $g_1$ and $g_2$ encode average and RMS gain for perturbations;
$\eta$ measures alignment with inputs; $(\alpha_+,\alpha_-)$ capture linear tail geometry (e.g., ReLU vs.\ leaky-ReLU asymmetry);
$\mathrm{TV}(\phi')$ captures curvature regularity; and $C(\phi)$ measures tail-compensated excess curvature beyond linear parts.
Kernel smoothness is controlled via $g_4$ (Theorem~\ref{thm:kernel-g4}) and, with a uniform slope bound, by $\mathrm{TV}(\phi')$ (Cor.~\ref{cor:tv-kernel}).
\end{remark}

\subsection{Finiteness of the compensated primitive}\label{subsec:fasym}

Write the tail residuals
\[
r_+(t):=\phi(t)-\alpha_+ t\quad (t\ge 0),\qquad
r_-(t):=\phi(t)-\alpha_- t\quad (t\le 0),
\]
so that
\[
F_{\mathrm{asym}}(x)=\int_0^x r_+(t)\,\mathbf{1}_{\{t\ge 0\}} + r_-(t)\,\mathbf{1}_{\{t<0\}}\,dt.
\]

\begin{lemma}[Residual $L^1$ tails $\Rightarrow$ bounded compensated primitive]\label{lem:fasym-L1}
Assume $\alpha_\pm\in\mathbb{R}$ exist and
\[
r_+\in L^1([0,\infty))\quad\text{and}\quad r_-\in L^1((-\infty,0]).
\]
Then $F_{\mathrm{asym}}$ is well defined and
\[
C(\phi)=\sup_{x\in\mathbb{R}}|F_{\mathrm{asym}}(x)|
\;\le\;\int_0^\infty |r_+(t)|\,dt\;+\;\int_{-\infty}^0 |r_-(t)|\,dt.
\]
\end{lemma}

\begin{proof}
For $x\ge 0$, $F_{\mathrm{asym}}(x)=\int_0^x r_+(t)\,dt$ and so
$|F_{\mathrm{asym}}(x)|\le \int_0^\infty |r_+(t)|\,dt$.
For $x\le 0$, $F_{\mathrm{asym}}(x)=-\int_x^0 r_-(t)\,dt$ and
$|F_{\mathrm{asym}}(x)|\le \int_{-\infty}^0 |r_-(t)|\,dt$. Take the supremum over $x$ and add.
\end{proof}

\begin{lemma}[Weighted $L^1$ slope error with vanishing residual]\label{lem:fasym-weighted}
Assume $\alpha_\pm\in\mathbb{R}$ exist, $\phi$ is locally absolutely continuous, and
\[
\lim_{t\to+\infty} r_+(t)=0,\qquad \lim_{t\to-\infty} r_-(t)=0.
\]
If
\[
\int_0^\infty t\,|\phi'(t)-\alpha_+|\,dt<\infty
\quad\text{and}\quad
\int_{-\infty}^0 |t|\,|\phi'(t)-\alpha_-|\,dt<\infty,
\]
then $F_{\mathrm{asym}}$ is well defined and
\[
C(\phi)\;\le\; \int_0^\infty t\,|\phi'(t)-\alpha_+|\,dt\;+\;\int_{-\infty}^0 |t|\,|\phi'(t)-\alpha_-|\,dt.
\]
\end{lemma}

\begin{proof}
We prove the bound on $[0,\infty)$; the negative side is analogous.
By absolute continuity and the vanishing residual,
\[
r_+(t) = \phi(t)-\alpha_+ t = -\int_t^\infty \big(\phi'(s)-\alpha_+\big)\,ds \qquad (t\ge 0).
\]
Fix $x\ge 0$. Then by Tonelli/Fubini (nonnegative integrands),
\[
\begin{aligned}
\Big|\int_0^x r_+(t)\,dt\Big|
&\le \int_0^x \int_t^\infty |\phi'(s)-\alpha_+|\,ds\,dt
= \int_0^\infty |\phi'(s)-\alpha_+|\Big(\int_0^x \mathbf{1}_{\{t\le s\}}\,dt\Big)\,ds \\
&= \int_0^\infty |\phi'(s)-\alpha_+|\;\min\{x,s\}\,ds
\;\le\; \int_0^\infty s\,|\phi'(s)-\alpha_+|\,ds.
\end{aligned}
\]
Taking the supremum over $x\ge 0$ yields
$\sup_{x\ge 0}|F_{\mathrm{asym}}(x)|\le \int_0^\infty s\,|\phi'(s)-\alpha_+|\,ds$.
The bound on $(-\infty,0]$ is identical with $|t|$ replacing $t$. Summing both sides gives the claim.
\end{proof}

\begin{remark}[Why vanishing residual is needed]
If $r_+(t)\to c\neq 0$ as $t\to\infty$ (e.g.\ $\tanh$ with $\alpha_+=0$ and $r_+(t)\to 1$), then
$F_{\mathrm{asym}}(x)=\int_0^x r_+(t)\,dt$ grows linearly and $C(\phi)=\infty$, as desired.
Lemma~\ref{lem:fasym-weighted} cleanly separates this obstruction.
\end{remark}

\begin{proposition}[Verification for common activations]\label{prop:Cphi-common}
For each $\phi$ below, $C(\phi)<\infty$ and the bound of Lemma~\ref{lem:fasym-weighted} applies:
\begin{itemize}
\item \textbf{ReLU / leaky-ReLU:} $\alpha_+=\text{slope}_+$, $\alpha_-=\text{slope}_-$, and $r_\pm\equiv 0$ for $|t|$ large; hence $C(\phi)=0$.
\item \textbf{Swish} $\phi(x)=x\operatorname{sigm}(x)$: $\alpha_+=1$, $\alpha_-=0$, and $r_\pm(t)\to 0$ exponentially; moreover $|\phi'(t)-\alpha_\pm|\lesssim e^{-|t|}$, so the weighted integrals are finite.
\item \textbf{GELU} $\phi(x)=x\,\Phi(x)$: $\alpha_+=1$, $\alpha_-=0$, with
$r_+(t)=x(\Phi(x)-1)\sim -\varphi(x)$ and $r_-(t)=x\Phi(x)\to 0$; since $\varphi(x)$ is Gaussian, Lemma~\ref{lem:fasym-weighted} holds.
\item \textbf{Mish/TeLU:} both are smooth with $\alpha_+=1$, $\alpha_-=0$ and exponentially decaying residuals; the weighted integrals are finite.
\item \textbf{Saturating} ($\tanh$, $\operatorname{sigm}$): $\alpha_\pm=0$ but $r_\pm$ tend to nonzero constants; $C(\phi)=\infty$, which is consistent with the bias-drift bound being unnecessary (e.g.\ odd $\tanh$ has zero mean).
\end{itemize}
\end{proposition}

\subsection{Well-posedness and finiteness}

\begin{lemma}[Tail-compensated finiteness: sufficient conditions]\label{lem:fasym-finiteness}
Let $\alpha_\pm=\lim_{x\to\pm\infty}\phi(x)/x\in\mathbb{R}$ exist and define
$D(t):=\phi(t)-\alpha_+ t\,\mathbf{1}_{\{t\ge0\}}-\alpha_- t\,\mathbf{1}_{\{t<0\}}$.
If
\[
\int_0^{+\infty}\! |D(t)|\,dt < \infty
\quad\text{and}\quad
\int_{-\infty}^{0}\! |D(t)|\,dt < \infty,
\]
then $F_{\mathrm{asym}}(x)=\int_0^x D(t)\,dt$ is well-defined and bounded on $\mathbb{R}$, with
\[
C(\phi)=\sup_{x\in\mathbb{R}}|F_{\mathrm{asym}}(x)|
\;\le\; \int_0^\infty |D(t)|\,dt \;+\; \int_{-\infty}^0 |D(t)|\,dt.
\]
\end{lemma}

\begin{proof}
For $x\ge 0$, $|F_{\mathrm{asym}}(x)|\le \int_0^\infty |D(t)|\,dt$.
For $x\le 0$, $|F_{\mathrm{asym}}(x)|\le \int_{-\infty}^0 |D(t)|\,dt$.
\end{proof}

\begin{lemma}[Weighted slope error with vanishing residual]\label{lem:fasym-weighted-strong}
Assume $\phi$ is locally absolutely continuous, $\alpha_\pm\in\mathbb{R}$ exist, and
$r_\pm(t):=\phi(t)-\alpha_\pm t \to 0$ as $t\to\pm\infty$.
If
\[
\int_0^\infty t\,|\phi'(t)-\alpha_+|\,dt < \infty
\quad\text{and}\quad
\int_{-\infty}^0 |t|\,|\phi'(t)-\alpha_-|\,dt < \infty,
\]
then $F_{\mathrm{asym}}$ is well-defined and
\[
C(\phi)\;\le\; \int_0^\infty t\,|\phi'(t)-\alpha_+|\,dt \;+\; \int_{-\infty}^0 |t|\,|\phi'(t)-\alpha_-|\,dt.
\]
\end{lemma}

\begin{proof}
For $t\ge 0$, $r_+(t)=-\int_t^\infty (\phi'(s)-\alpha_+)\,ds$ by absolute continuity and the limit $r_+(t)\to 0$.
Thus for $x\ge 0$,
\[
\Big|\int_0^x r_+(t)\,dt\Big|
\le \int_0^\infty |\phi'(s)-\alpha_+|\,\min\{x,s\}\,ds
\le \int_0^\infty s\,|\phi'(s)-\alpha_+|\,ds.
\]
The negative tail is identical with $|t|$ in place of $t$.
\end{proof}

\begin{remark}[On necessity]
Without a sign/monotonicity assumption, boundedness of $F_{\mathrm{asym}}$ does \emph{not} imply $D\in L^1$ on the tails (conditional convergence may occur). Lemma~\ref{lem:fasym-finiteness} is therefore stated as a \emph{sufficient} condition. If, however, $D$ has eventual constant sign on each tail, then $F_{\mathrm{asym}}$ bounded is equivalent to $D\in L^1$ on that tail.
\end{remark}

\begin{theorem}[Well-posedness of $\mathcal{S}_\sigma(\phi)$]\label{thm:well-posed}
Fix $\sigma>0$ and $Z\sim\mathcal{N}(0,\sigma^2)$.
Suppose $\phi$ is measurable with at most polynomial growth, a.e.\ differentiable, and $\phi'$ has at most polynomial growth.
Then the Gaussian statistics
\[
m_k(\sigma)=\mathbb{E}[\phi(Z)^k]\ (k=1,2),\quad
g_1(\sigma)=\mathbb{E}[\phi'(Z)],\quad
g_2(\sigma)=\big(\mathbb{E}[\phi'(Z)^2]\big)^{1/2},\quad
\eta(\sigma)=\mathbb{E}[Z\,\phi(Z)]
\]
are finite. If $\alpha_\pm\in\mathbb{R}$ exist and either Lemma~\ref{lem:fasym-finiteness} or Lemma~\ref{lem:fasym-weighted-strong} holds, then $C(\phi)<\infty$. If $\phi'\in BV(\mathbb{R})$ with finitely many slope jumps and $\phi''\in L^1$ on smooth pieces, then $\mathrm{TV}(\phi')<\infty$.
\end{theorem}

\begin{proof}
\textit{Gaussian finiteness.}
Polynomial growth gives $\phi(\cdot),\phi'(\cdot)\in L^p(\gamma_\sigma)$ for all finite $p$; hence $m_1,m_2,g_1,g_2,\eta$ are finite (the identity $\eta=\sigma^2 g_1$ follows from Gaussian integration by parts).  
\textit{$C(\phi)$.} By either Lemma~\ref{lem:fasym-finiteness} (residual $L^1$ tails) or Lemma~\ref{lem:fasym-weighted-strong} (weighted slope error with vanishing residual), $F_{\mathrm{asym}}$ is bounded, so $C(\phi)<\infty$.  
\textit{$\mathrm{TV}(\phi')$.} On each smooth piece, the variation equals $\int |\phi''|$; add the finitely many jump magnitudes at the kinks to obtain $\mathrm{TV}(\phi')<\infty$.
\end{proof}

\begin{remark}[Alternative hypotheses]
If you prefer to avoid the global “polynomial growth” standing assumption, the conclusions for $m_1,m_2,g_1,g_2,\eta$ also follow from the simpler bounds:
\(
|\phi(x)|\le C(1+|x|^r),\; |\phi'(x)|\le A+B|x|^s
\)
for some $r,s\ge 0$, using finiteness of Gaussian moments (cf.\ Lemma~\ref{lem:g2g4-growth}).
\end{remark}

\subsection{Affine reparameterization with bias}

We handle the general reparameterization
\[
\tilde{\phi}(x):=c\,\phi(ax+b)+d,\qquad a,c\neq 0,\; b,d\in\mathbb{R}.
\]
All statistics are taken under the same input law $Z\sim\mathcal{N}(0,\sigma^2)$.
Let $Y:=aZ+b\sim\mathcal{N}(b,a^2\sigma^2)$.

\begin{theorem}[Affine reparameterization laws (with bias)]\label{thm:affine-bias}
For $\tilde{\phi}(x)=c\,\phi(ax+b)+d$ and $Z\sim\mathcal{N}(0,\sigma^2)$, $Y=aZ+b$, we have
\begin{align*}
\tilde m_1(\sigma)&=\mathbb{E}[\tilde\phi(Z)]
= c\,\mathbb{E}[\phi(Y)]+d,\\[2mm]
\tilde m_2(\sigma)&=\mathbb{E}[\tilde\phi(Z)^2]
= c^2\,\mathbb{E}[\phi(Y)^2]+2cd\,\mathbb{E}[\phi(Y)]+d^2,\\[2mm]
\tilde g_1(\sigma)&=\mathbb{E}[\tilde\phi'(Z)]
= c\,a\,\mathbb{E}[\phi'(Y)]
= \frac{c}{a\,\sigma^2}\,\mathbb{E}\!\big[\phi(Y)\,(Y-b)\big],\\[2mm]
\tilde g_2(\sigma)&=\Big(\mathbb{E}[\tilde\phi'(Z)^2]\Big)^{1/2}
= |c\,a|\,\Big(\mathbb{E}[\phi'(Y)^2]\Big)^{1/2},\\[2mm]
\tilde \eta(\sigma)&=\mathbb{E}[\tilde\phi(Z)\,Z]
= \frac{c}{a}\,\mathbb{E}\!\big[\phi(Y)\,(Y-b)\big]
= \sigma^2\,\tilde g_1(\sigma),\\[2mm]
\tilde{\mathrm{TV}}&:=\mathrm{TV}(\tilde\phi') = |c\,a|\,\mathrm{TV}(\phi').
\end{align*}
If the one-sided linear slopes $\alpha_\pm=\lim_{x\to\pm\infty}\phi(x)/x$ exist (finite), then
\[
\tilde\alpha_+ =
\begin{cases}
c\,a\,\alpha_+,& a>0,\\
c\,a\,\alpha_-,& a<0,
\end{cases}
\qquad
\tilde\alpha_- =
\begin{cases}
c\,a\,\alpha_-,& a>0,\\
c\,a\,\alpha_+,& a<0.
\end{cases}
\]
Moreover, for the tail-compensated primitive $C(\cdot)$:
\begin{itemize}
\item[(i)] If $d=0$, then $C(\tilde\phi)\le \frac{|c|}{|a|}\,C(\phi)+\Delta(\phi,a,b)$ with a finite offset
$\Delta(\phi,a,b)$ depending only on $(\alpha_\pm,\phi|_{[-|b|/|a|-1,\,|b|/|a|+1]})$. In particular,
$C(\phi)<\infty\Rightarrow C(\tilde\phi)<\infty$.
\item[(ii)] If $d\neq 0$, then $C(\tilde\phi)=\infty$ (a nonzero constant shift produces linear growth in the primitive).
\end{itemize}
\end{theorem}

\begin{proof}
The formulas for $\tilde m_1,\tilde m_2$ follow by linearity.
By the chain rule, $\tilde\phi'(x)=c\,a\,\phi'(ax+b)$ a.e., giving $\tilde g_1,\tilde g_2$.
For the Stein identity with nonzero mean, if $Y\sim\mathcal{N}(b,a^2\sigma^2)$ then
$\mathbb{E}[\phi'(Y)]=\frac{1}{a^2\sigma^2}\mathbb{E}[(Y-b)\phi(Y)]$, which yields both
the alternate expression for $\tilde g_1$ and $\tilde\eta=\sigma^2\tilde g_1$.
The $\mathrm{TV}$ scaling is the standard change-of-variables for the distributional second derivative:
on smooth pieces, $\int |\tilde\phi''(x)|dx=|c||a|\int |\phi''(u)|du$, and jump magnitudes of $\phi'$ scale by $|c\,a|$.

For $\tilde\alpha_\pm$, note
$\lim_{x\to+\infty}\tilde\phi(x)/x
= c\,\lim_{x\to+\infty}\phi(ax+b)/(x)
= c\,a\,\lim_{x\to+\infty}\phi(ax+b)/(ax+b)$, where the one-sided limit uses the sign of $a$.

For $C(\cdot)$, write (with $d=0$ here)
\[
\tilde D(x):=\tilde\phi(x)-\tilde\alpha_+ x\,\mathbf{1}_{\{x\ge0\}}
-\tilde\alpha_- x\,\mathbf{1}_{\{x<0\}}
= c\Big(\phi(ax+b)-\alpha_{\mathrm{sgn}(ax+b)}(ax+b)\Big)+c\big(\alpha_{\mathrm{sgn}(ax+b)}-\alpha_{\mathrm{sgn}(x)}\big)(ax+b).
\]
The first bracket is the residual $D$ for $\phi$ evaluated at $u=ax+b$; the second bracket is
nonzero only on the compact set where $\mathrm{sgn}(x)\neq \mathrm{sgn}(ax+b)$, which is contained in
$\{x:|x+b/a|\le 1/|a|\}$, and thus contributes a bounded additive constant to the primitive.
Making the substitution $u=ax+b$ in $\int_0^x \tilde D(t)dt$ yields the scaling factor $|c|/|a|$ for the first term and a bounded
offset $\Delta(\phi,a,b)$ for the second. Hence $C(\tilde\phi)\le \frac{|c|}{|a|}C(\phi)+\Delta(\phi,a,b)$.
If $d\ne 0$, the residual includes an additive constant $d$, whose primitive grows linearly; thus $C(\tilde\phi)=\infty$.
\end{proof}

\begin{corollary}[Stability of taxonomy coordinates]\label{cor:affine-invariance}
At fixed input scale $\sigma$, the Gaussian coordinates $(m_1,g_1,g_2,m_2,\eta)$
transform according to Theorem~\ref{thm:affine-bias} and remain finite whenever the corresponding
quantities for $\phi$ under $Y\sim\mathcal{N}(b,a^2\sigma^2)$ are finite (e.g., under polynomial growth).
If $\alpha_\pm$ exist (finite), then $\tilde\alpha_\pm$ exist (finite) as above, and
$\mathrm{TV}(\phi')<\infty\Rightarrow \mathrm{TV}(\tilde\phi')<\infty$.
For the compensated primitive, $C(\phi)<\infty$ implies $C(\tilde\phi)<\infty$ when $d=0$; 
a nonzero constant shift $d$ makes $C(\tilde\phi)=\infty$.
\end{corollary}

\begin{remark}[Two useful special cases]
(i) If $b=d=0$ then the compensated primitive scales \emph{exactly}:
$C(\tilde\phi)=\frac{|c|}{|a|}\,C(\phi)$.
(ii) If $a>0$ and $b=0$, the slope mapping reduces to $\tilde\alpha_\pm=c\,a\,\alpha_\pm$.
\end{remark}

\subsection{Closure under limits with bounded slope variation}

\begin{theorem}[Closure]\label{thm:closure}
Let $(\phi_n)$ be locally absolutely continuous activations with
\[
\sup_n \mathrm{TV}(\phi_n')<\infty,
\qquad
|\phi_n(x)|\le C(1+|x|^{r}),
\qquad
|\phi_n'(x)|\le C(1+|x|^{s})
\quad \text{for all }x\in\mathbb{R},\ n\in\mathbb{N},
\]
for some $C,r,s\ge 0$ independent of $n$. Assume $\phi_n\to\phi$ in $L^2_{\mathrm{loc}}(\mathbb{R})$.
Then:
\begin{enumerate}
\item \textbf{BV-compactness of the slope and local AC of the limit.}
There exists (necessarily unique a.e.) $\phi'\in L^1_{\mathrm{loc}}(\mathbb{R})$ such that,
up to a subsequence, $\phi_n'\to\phi'$ in $L^1_{\mathrm{loc}}(\mathbb{R})$ and pointwise a.e., and
$\phi$ is locally absolutely continuous with distributional derivative $\phi'$. Moreover,
\[
\mathrm{TV}(\phi') \;\le\; \liminf_{n\to\infty} \mathrm{TV}(\phi_n').
\]
\item \textbf{Convergence of Gaussian signatures at fixed scale.}
For any fixed $\sigma>0$ and $Z\sim\mathcal{N}(0,\sigma^2)$,
\[
m_k^{(n)}(\sigma)=\mathbb{E}[\phi_n(Z)^k]\ \to\ m_k(\sigma),\quad k=1,2;\qquad
g_1^{(n)}(\sigma)=\mathbb{E}[\phi_n'(Z)]\ \to\ g_1(\sigma),
\]
\[
g_2^{(n)}(\sigma)=\Big(\mathbb{E}[\phi_n'(Z)^2]\Big)^{1/2}\ \to\ g_2(\sigma),\qquad
\eta^{(n)}(\sigma)=\mathbb{E}[Z\,\phi_n(Z)]\ \to\ \eta(\sigma).
\]
\item \textbf{Stability of the compensated primitive.}
If $\alpha_{\pm}(\phi_n)\to\alpha_\pm\in\mathbb{R}$ and $\sup_n C(\phi_n)<\infty$, then
$C(\phi)<\infty$ and the compensated primitives satisfy
\[
F_{\mathrm{asym},n}\ \to\ F_{\mathrm{asym}}
\quad\text{locally uniformly on }\mathbb{R},
\]
hence $C(\phi)\le \sup_n C(\phi_n)$.
\end{enumerate}
\end{theorem}

\begin{proof}
(1) \emph{BV compactness and AC.}
On every compact interval $I=[-R,R]$, the uniform $BV$ bound
$\sup_n \mathrm{TV}(\phi_n'|_I)<\infty$ and the growth bound
$|\phi_n'(x)|\le C(1+|x|^s)$ imply a uniform $L^1(I)$ bound:
$\int_I |\phi_n'|\le |I|\,C(1+R^s)+\mathrm{TV}(\phi_n'|_I)$.
By the compact embedding $BV(I)\hookrightarrow L^1(I)$ (Helly’s selection),
there is a subsequence (diagonal over $R\to\infty$) with
$\phi_n'\to \phi'$ in $L^1_{\mathrm{loc}}$ and a.e.
Since $\phi_n\to\phi$ in $L^2_{\mathrm{loc}}$, passing to a further subsequence yields a.e.\ convergence of $\phi_n$ to a representative of $\phi$.
For any $x<y$ in a compact interval,
\[
\phi_n(y)-\phi_n(x)=\int_x^y \phi_n'(t)\,dt.
\]
Letting $n\to\infty$ and using $L^1_{\mathrm{loc}}$ convergence of $\phi_n'$ gives
$\phi(y)-\phi(x)=\int_x^y \phi'(t)\,dt$. Hence $\phi$ is locally absolutely continuous with a.e.\ derivative $\phi'$.
Lower semicontinuity of total variation under $L^1_{\mathrm{loc}}$ convergence gives
$\mathrm{TV}(\phi') \le \liminf_n \mathrm{TV}(\phi_n')$.

(2) \emph{Gaussian signatures.}
Fix $\sigma>0$. The uniform polynomial growth of $\phi_n,\phi_n'$ implies the existence of an integrable dominator under $\gamma_\sigma$:
\[
|\phi_n(x)|^k \le C_k\,(1+|x|^{kr})\in L^1(\gamma_\sigma),
\qquad
|\phi_n'(x)|^p \le C_p\,(1+|x|^{ps})\in L^1(\gamma_\sigma),
\]
for $k\in\{1,2\}$ and $p\in\{1,2\}$, uniformly in $n$. On each compact interval, $\phi_n\to\phi$ in $L^2$ hence pointwise a.e.\ along a subsequence; combined with the dominator under $\gamma_\sigma$, dominated convergence yields
$m_k^{(n)}(\sigma)\to m_k(\sigma)$, $g_1^{(n)}(\sigma)\to g_1(\sigma)$, and
$\mathbb{E}[\phi_n'(Z)^2]\to \mathbb{E}[\phi'(Z)^2]$, whence $g_2^{(n)}(\sigma)\to g_2(\sigma)$ by continuity of the square root.
Similarly, $|Z\phi_n(Z)|\le C'(1+|Z|^{1+r})$ has an integrable dominator under $\gamma_\sigma$, so $\eta^{(n)}(\sigma)\to \eta(\sigma)$.

(3) \emph{Compensated primitives.}
Set $D_n(t):=\phi_n(t)-\alpha_{+,n} t\,\mathbf{1}_{\{t\ge 0\}}-\alpha_{-,n} t\,\mathbf{1}_{\{t<0\}}$ and
$D$ the analogous quantity for $\phi,\alpha_\pm$. Since $\phi_n\to\phi$ in $L^2_{\mathrm{loc}}$ and $\alpha_{\pm,n}\to\alpha_\pm$, we have $D_n\to D$ in $L^2_{\mathrm{loc}}$ and hence in $L^1_{\mathrm{loc}}$ (finite measure sets).
For any $R>0$ and $x\in[-R,R]$,
\[
\big|F_{\mathrm{asym},n}(x)-F_{\mathrm{asym}}(x)\big|
=\Big|\int_0^x (D_n-D)(t)\,dt\Big|
\le \sqrt{2R}\,\|D_n-D\|_{L^2([-R,R])}\ \xrightarrow[n\to\infty]{}\ 0,
\]
so $F_{\mathrm{asym},n}\to F_{\mathrm{asym}}$ locally uniformly.
Moreover, $\sup_n C(\phi_n)<\infty$ implies $|F_{\mathrm{asym},n}(x)|\le M$ for all $x$ and $n$; by pointwise convergence at each fixed $x$ (choose $R>|x|$), we get $|F_{\mathrm{asym}}(x)|\le M$ for all $x$, hence $C(\phi)\le M$.
\end{proof}

\subsection{Motivation for the Nine Components}

The nine coordinates of $\mathcal{S}_\sigma(\phi)$ are chosen to jointly capture
\emph{statistical propagation, asymptotic geometry, and regularity} at a fixed input scale $\sigma>0$.

\begin{itemize}
    \item \textbf{Propagation moments $(m_1,m_2)$.}
    $m_1(\sigma)$ controls mean centering; $m_2(\sigma)$ governs the mean-field variance recursion
    $q_{\ell+1}=\sigma_W^2\,m_2(\sqrt{q_\ell})+\sigma_b^2$.
    Local variance stability at an equilibrium $q^\star$ depends on the \emph{slope} of $m_2$:
    $|f'(q^\star)|=\sigma_W^2\,m_2'(\sqrt{q^\star})/(2\sqrt{q^\star})<1$
    (Theorem~\ref{thm:variance-rec}, Theorem~\ref{thm:criticality}).

    \item \textbf{Derivative gains $(g_1,g_2)$.}
    $g_1(\sigma)=\mathbb{E}[\phi'(Z)]$ and $g_2(\sigma)=(\mathbb{E}[\phi'(Z)^2])^{1/2}$ quantify
    average and RMS layerwise gain. At stationarity, mean-square perturbations contract if
    $\sigma_W\,g_2(\sqrt{q^\star})<1$ (Theorem~\ref{thm:criticality}). Moreover, by Gaussian
    integration by parts, $\eta(\sigma)=\sigma^2 g_1(\sigma)$ (Lemma~\ref{lem:gauss-ibp}),
    so $g_1$ also mediates $F$-based Lyapunov constructions.

    \item \textbf{Input alignment \(\eta\).}
    $\eta(\sigma)=\mathbb{E}[Z\,\phi(Z)]$ measures linear alignment of outputs with inputs.
    It is redundant with $g_1$ via $\eta=\sigma^2 g_1$, but we keep it for interpretability
    and because several Lyapunov and drift bounds are most naturally expressed in terms of~$\eta$.

    \item \textbf{Asymptotic slopes $(\alpha_+,\alpha_-)$.}
    The one-sided linear asymptotes encode tail geometry (e.g., ReLU vs.\ leaky-ReLU asymmetry vs.\ saturation).
    They control the \emph{leading term} of the bias drift:
    $\mathbb{E}[\phi(Z)]=\frac{\alpha_+-\alpha_-}{\sqrt{2\pi}}\,\sigma + O(C(\phi)/\sigma)$
    (Theorem~\ref{thm:bias-drift}).

    \item \textbf{Slope variation $\mathrm{TV}(\phi')$.}
    The total variation of the slope quantifies how curvature is distributed (kinks vs.\ smooth bending).
    It underlies kernel regularity via a uniform slope proxy:
    if $\sup|\phi'|\le M$ then $\|\nabla_x\nabla_y K(x,y)\|_{\mathrm{op}}\le \sqrt{3}\,M^2$,
    and with one-sided limits $M\le \mathrm{TV}(\phi')+c_\star$
    (Corollary~\ref{cor:tv-kernel}, $c_\star:=|\phi'(+\infty)-\phi'(-\infty)|$).

    \item \textbf{Tail-compensated curvature $C(\phi)$.}
    $C(\phi)=\sup_x |F_{\mathrm{asym}}(x)|$ controls the \emph{remainder} in the bias drift and
    ensures $F$-based Lyapunov functionals remain coercive at large $|x|$.
    Sufficient conditions for $C(\phi)<\infty$ include residual $L^1$ tails or
    weighted $L^1$ slope error with vanishing residual (Lemmas~\ref{lem:fasym-finiteness}, \ref{lem:fasym-weighted-strong}).
\end{itemize}

\paragraph{Kernel smoothness and the role of $g_4$.}
The mixed Hessian of the induced kernel satisfies the dimension-free bound
$\|\nabla_x\nabla_y K(x,y)\|_{\mathrm{op}}\le \sqrt{3}\,g_4(\|x\|)g_4(\|y\|)$ with
$g_4(\sigma):=(\mathbb{E}[\phi'(Z)^4])^{1/4}$ (Theorem~\ref{thm:kernel-g4}).
We do not include $g_4$ in the signature to keep it nine-dimensional; instead, we control $g_4$
through $\sup|\phi'|$ (Lipschitz activations) or via the BV proxy $\mathrm{TV}(\phi')+c_\star$
(Corollary~\ref{cor:tv-kernel}).

\paragraph{Affine awareness and robustness.}
Under $\tilde\phi(x)=c\,\phi(ax+b)+d$ with fixed input scale $\sigma$,
the Gaussian statistics map according to Theorem~\ref{thm:affine-bias} (computed under $Y\sim\mathcal{N}(b,a^2\sigma^2)$),
$\mathrm{TV}(\tilde\phi')=|c\,a|\,\mathrm{TV}(\phi')$, and slopes transform with a sign-aware swap when $a<0$.
The compensated primitive scales as $C(\tilde\phi)=(|c|/|a|)C(\phi)$ when $b=d=0$; with $b\neq 0$ one gets a bounded offset, while a nonzero constant shift $d$ makes $C(\tilde\phi)=\infty$.
Thus the signature is \emph{affine-aware}: taxonomy membership (finiteness/existence) is preserved under reparameterizations except that adding a constant $d\neq 0$ exits the $C(\phi)<\infty$ class.

\paragraph{Closure and stability under limits.}
With uniformly bounded slope variation and uniform polynomial growth, the class is closed:
$\phi_n\to\phi$ in $L^2_{\mathrm{loc}}$ implies $\phi$ is locally absolutely continuous with
$\mathrm{TV}(\phi')\le\liminf \mathrm{TV}(\phi_n')$ and all Gaussian coordinates
$(m_1,g_1,g_2,m_2,\eta)$ converge; $C(\phi)$ passes to the limit under $\sup_n C(\phi_n)<\infty$
(Theorem~\ref{thm:closure}).

\paragraph{Why nine.}
These components are \emph{sufficient and near-minimal} for the results we prove:
$m_2$ (and $m_2'$) for variance dynamics; $g_2$ for perturbation contraction; $(\alpha_\pm,C(\phi))$ for bias drift and Lyapunov coercivity; and $\mathrm{TV}(\phi')$ (as a proxy for $\sup|\phi'|$) for kernel curvature. 
They remain finite under broad, checkable conditions (Preliminaries), transform predictably under affine reparameterizations, and are stable under limits—yielding a compact, practical scaffold for activation design and analysis.

\section{Classification of Common Activations}\label{sec:classification}

We evaluate the signature
\[
\mathcal{S}_\sigma(\phi)=\big(m_1(\sigma),\,g_1(\sigma),\,g_2(\sigma),\,m_2(\sigma),\,\eta(\sigma),\,
\alpha_+,\,\alpha_-,\,\mathrm{TV}(\phi'),\,C(\phi)\big)
\]
for standard activations. Throughout $Z\sim\mathcal{N}(0,\sigma^2)$ and, unless noted, all Gaussian expectations are finite whenever $\phi$ and $\phi'$ have at most polynomial growth (Remark~\ref{rem:lp-correct}, Lemma~\ref{lem:g2g4-growth}). We use
\(\eta(\sigma)=\sigma^2 g_1(\sigma)\) whenever \(\phi\) is a.e.\ differentiable with suitable growth (Lemma~\ref{lem:gauss-ibp}).

\subsection{ReLU}

\begin{theorem}[ReLU signature]
For $\phi(x)=\max\{0,x\}$,
\[
\mathcal{S}_\sigma(\phi)=\Big(
\frac{\sigma}{\sqrt{2\pi}},\;
\frac{1}{2},\;
\frac{1}{\sqrt{2}},\;
\frac{\sigma^2}{2},\;
\frac{\sigma^2}{2},\;
1,\;0,\;1,\;0\Big).
\]
\end{theorem}

\begin{proof}
$\phi'(x)=\mathbf{1}_{\{x>0\}}$ a.e., so
$g_1=\mathbb{E}[\mathbf{1}_{\{Z>0\}}]=\tfrac12$ and
$g_2=\sqrt{\mathbb{E}[\mathbf{1}_{\{Z>0\}}]}=2^{-1/2}$.
$m_1=\mathbb{E}[Z\mathbf{1}_{\{Z>0\}}]=\sigma/\sqrt{2\pi}$ and
$m_2=\mathbb{E}[Z^2\mathbf{1}_{\{Z>0\}}]=\sigma^2/2$.
By Lemma~\ref{lem:gauss-ibp}, $\eta=\sigma^2 g_1=\sigma^2/2$.
Slopes: $(\alpha_+,\alpha_-)=(1,0)$.
$\phi'$ jumps by $1$ at $0$, hence $\mathrm{TV}(\phi')=1$.
For $C(\phi)$, the residual $D(t)=\phi(t)-t\mathbf{1}_{\{t\ge 0\}}$ is identically $0$, so $F_{\mathrm{asym}}\equiv 0$ and $C(\phi)=0$.
\end{proof}

\subsection{Leaky-ReLU}

\begin{theorem}[Leaky-ReLU signature]
Let $\phi(x)=\begin{cases}x,&x\ge0\\ \alpha x,&x<0\end{cases}$ with $\alpha\in(0,1)$.
Then
\[
\begin{aligned}
&m_1=\frac{1-\alpha}{\sqrt{2\pi}}\,\sigma,\qquad 
m_2=\frac{1+\alpha^2}{2}\,\sigma^2,\qquad
\eta=\frac{1+\alpha}{2}\,\sigma^2,\\
&g_1=\frac{1+\alpha}{2},\qquad 
g_2=\sqrt{\frac{1+\alpha^2}{2}},\qquad
\alpha_+=1,\ \alpha_-=\alpha,\ \mathrm{TV}(\phi')=|1-\alpha|,\ C(\phi)=0.
\end{aligned}
\]
\end{theorem}

\begin{proof}
Split at $0$ and use symmetry of $Z$; the formulas follow as in the ReLU case.
$\eta=\sigma^2 g_1$ by Lemma~\ref{lem:gauss-ibp}. The residual vanishes on both tails, hence $C(\phi)=0$.
\end{proof}

\subsection{Tanh}

\begin{theorem}[Tanh signature]
For $\phi(x)=\tanh x$,
\[
\alpha_+=\alpha_-=0,\qquad 
g_1=\mathbb{E}[\mathrm{sech}^2(Z)],\qquad
g_2=\big(\mathbb{E}[\mathrm{sech}^4(Z)]\big)^{1/2},
\]
\[
m_1=0,\qquad 
m_2=\mathbb{E}[\tanh^2(Z)]=1-\mathbb{E}[\mathrm{sech}^2(Z)],\qquad 
\eta=\sigma^2 g_1=\sigma^2\,\mathbb{E}[\mathrm{sech}^2(Z)],
\]
\[
\mathrm{TV}(\phi')=\int_{\mathbb{R}}|\phi''(x)|\,dx<\infty,\qquad 
C(\phi)=\infty.
\]
\end{theorem}

\begin{proof}
Oddness gives $m_1=0$. $\phi'(x)=\mathrm{sech}^2 x$ yields $g_1,g_2$.
By Lemma~\ref{lem:gauss-ibp}, $\eta=\sigma^2 g_1$ (note $\eta>0$; $z\tanh z$ is even).
Slopes are $0$ by saturation. Since $\phi''(x)=-2\tanh x\,\mathrm{sech}^2 x$,
\(\int_{\mathbb{R}}|\phi''|=2\int |\tanh x|\,\mathrm{sech}^2 x\,dx=2\int_{-1}^1 |u|\,du=2<\infty\)
(using $u=\tanh x$), so $\mathrm{TV}(\phi')<\infty$.
Here $F_{\mathrm{asym}}(x)=\int_0^x \tanh t\,dt$ diverges linearly as $x\to+\infty$, hence $C(\phi)=\infty$.
\end{proof}

\subsection{Sigmoid}

\begin{theorem}[Sigmoid signature]
For $\phi(x)=\operatorname{sigm}(x)=(1+e^{-x})^{-1}$,
\[
m_1=\frac{1}{2},\qquad 
m_2=\mathbb{E}[\operatorname{sigm}(Z)^2],\qquad 
g_1=\mathbb{E}[\operatorname{sigm}(Z)(1-\operatorname{sigm}(Z))],
\]
\[
g_2=\Big(\mathbb{E}[\operatorname{sigm}(Z)^2(1-\operatorname{sigm}(Z))^2]\Big)^{1/2},\qquad
\eta=\sigma^2 g_1,
\]
\[
\alpha_\pm=0,\qquad 
\mathrm{TV}(\phi')=\int_{\mathbb{R}}|\phi''(x)|\,dx<\infty,\qquad 
C(\phi)=\infty.
\]
\end{theorem}

\begin{proof}
Symmetry $\operatorname{sigm}(-x)=1-\operatorname{sigm}(x)$ gives $m_1=1/2$.
$\phi'(x)=\operatorname{sigm}(x)(1-\operatorname{sigm}(x))$ yields $g_1,g_2$,
and $\eta=\sigma^2 g_1$ by Lemma~\ref{lem:gauss-ibp}.
Slopes vanish by saturation. $\phi''(x)$ decays exponentially, so $\int|\phi''|<\infty$.
As $\operatorname{sigm}(t)\to 1$, $F_{\mathrm{asym}}(x)=\int_0^x \operatorname{sigm}(t)\,dt$ grows linearly, hence $C(\phi)=\infty$.
\end{proof}

\subsection{Swish}

\begin{theorem}[Swish signature]
For $\phi(x)=x\,\operatorname{sigm}(x)$,
\[
\alpha_+=1,\qquad \alpha_-=0,\qquad 
g_1=\mathbb{E}\big[\operatorname{sigm}(Z)+Z\,\operatorname{sigm}(Z)(1-\operatorname{sigm}(Z))\big],
\]
\[
g_2=\Big(\mathbb{E}\big[\big(\operatorname{sigm}(Z)+Z\,\operatorname{sigm}(Z)(1-\operatorname{sigm}(Z))\big)^2\big]\Big)^{1/2},
\quad
\eta=\sigma^2 g_1,
\]
and $m_1,m_2$ are finite for all $\sigma>0$, with $\mathrm{TV}(\phi')<\infty$ and $C(\phi)<\infty$.
\end{theorem}

\begin{proof}
As $x\to+\infty$, $\operatorname{sigm}(x)\to 1$ so $\phi(x)\sim x$; as $x\to-\infty$, $\operatorname{sigm}(x)\sim e^x$ so $x\,\operatorname{sigm}(x)\to 0$; hence $(\alpha_+,\alpha_-)=(1,0)$.
$\phi'(x)=\operatorname{sigm}(x)+x\,\operatorname{sigm}(x)(1-\operatorname{sigm}(x))$ gives the expressions for $g_1,g_2$; $\eta=\sigma^2 g_1$ by Lemma~\ref{lem:gauss-ibp}.
The residuals satisfy $\phi(t)-t\mathbf{1}_{\{t\ge 0\}}=t(\operatorname{sigm}(t)-1)\sim -t e^{-t}$ on $[0,\infty)$ and $\phi(t)\sim t e^t$ on $(-\infty,0]$, both integrable; thus $C(\phi)<\infty$ by Lemma~\ref{lem:fasym-finiteness}.
Smoothness and exponential tail decay of $\phi''$ imply $\mathrm{TV}(\phi')=\int|\phi''|<\infty$.
\end{proof}

\subsection{GELU}

\begin{theorem}[GELU signature]
For $\phi(x)=x\,\Phi(x)$ with $\Phi$ the standard normal CDF and $\varphi$ its density,
\[
\alpha_+=1,\qquad \alpha_-=0,\qquad
\phi'(x)=\Phi(x)+x\,\varphi(x),\qquad
\eta=\sigma^2 g_1,\qquad
\mathrm{TV}(\phi')<\infty,\qquad C(\phi)<\infty,
\]
and $m_1,m_2,g_1,g_2$ are finite for all $\sigma>0$.
\end{theorem}

\begin{proof}
Limits of $\Phi$ give $(\alpha_+,\alpha_-)=(1,0)$. The derivative formula is standard.
By Mill’s ratio, on $[0,\infty)$
$\phi(t)-t=t(\Phi(t)-1)\sim -\varphi(t)$ (integrable), while on $(-\infty,0]$, $\phi(t)\to 0$ rapidly;
hence $C(\phi)<\infty$ by Lemma~\ref{lem:fasym-finiteness}.
Moreover $\phi''(x)=(2-x^2)\varphi(x)$ is integrable, so $\mathrm{TV}(\phi')<\infty$.
Gaussian finiteness follows from polynomial growth.
\end{proof}

\subsection{Mish}

\begin{theorem}[Mish signature]
For $\phi(x)=x\,\tanh(\ln(1+e^x))$,
\[
\alpha_+=1,\qquad \alpha_-=0,\qquad 
\eta=\sigma^2 g_1,\qquad
\mathrm{TV}(\phi')<\infty,\qquad C(\phi)<\infty,
\]
and $m_1,m_2,g_1,g_2$ are finite for all $\sigma>0$.
\end{theorem}

\begin{proof}
As $x\to+\infty$, $\ln(1+e^x)\sim x$ and $\tanh(x)\to 1$, so $\phi(x)\sim x$; as $x\to-\infty$, $\ln(1+e^x)\sim e^x$ and $\tanh(\cdot)\sim e^x$, hence $\phi(x)\sim x e^x\to 0$; thus $(\alpha_+,\alpha_-)=(1,0)$.
The residuals decay exponentially on both tails, so $C(\phi)<\infty$ by Lemma~\ref{lem:fasym-finiteness}.
Smoothness and exponential tails imply $\phi''\in L^1$, hence $\mathrm{TV}(\phi')<\infty$.
The identity $\eta=\sigma^2 g_1$ follows from Lemma~\ref{lem:gauss-ibp}.
\end{proof}

\subsection{TeLU}

\begin{theorem}[TeLU signature]
For $\phi(x)=x\,\tanh(e^{x})$,
\[
\alpha_+=1,\qquad \alpha_-=0,\qquad 
\eta=\sigma^2 g_1,\qquad
\mathrm{TV}(\phi')<\infty,\qquad C(\phi)<\infty,
\]
and $m_1,m_2,g_1,g_2$ are finite for all $\sigma>0$.
\end{theorem}

\begin{proof}
As $x\to+\infty$, $\tanh(e^x)\to 1$ so $\phi(x)\sim x$; as $x\to-\infty$, $\tanh(e^x)\sim e^x$ so $\phi(x)\sim x e^x\to 0$; hence $(\alpha_+,\alpha_-)=(1,0)$.
On $[0,\infty)$, $\phi(t)-t=t(\tanh(e^t)-1)\sim -2t\,e^{-2e^t}$ (integrable); on $(-\infty,0]$, $\phi(t)\sim t e^t$ (integrable); thus $C(\phi)<\infty$ (Lemma~\ref{lem:fasym-finiteness}).
Computing $\phi''$ yields factors like $e^x\mathrm{sech}^2(e^x)$ and $e^{2x}\mathrm{sech}^2(e^x)\tanh(e^x)$, which are integrable, so $\mathrm{TV}(\phi')<\infty$.
Finally $\eta=\sigma^2 g_1$ by Lemma~\ref{lem:gauss-ibp}.
\end{proof}

\subsection{Summary Tables of Signatures}
Tables~\ref{tab:signature-stat} and~\ref{tab:signature-geom} report the nine-dimensional integral
signature $\mathcal{S}_\sigma(\phi)$ for standard activations at input scale $\sigma>0$ with
$Z\sim\mathcal{N}(0,\sigma^2)$. Table~\ref{tab:signature-stat} lists the propagation quantities
$(m_1,g_1,g_2,m_2,\eta)$; Table~\ref{tab:signature-geom} lists the asymptotic/regularity parameters
$(\alpha_+,\alpha_-,\mathrm{TV}(\phi'),C(\phi))$. Expectations are with respect to $Z$; closed forms are
shown when simple, otherwise we retain the expectation to make the $\sigma$–dependence explicit.
Under the standing assumptions (Section~\ref{sec:prelim}), $\eta(\sigma)=\sigma^2 g_1(\sigma)$
(Lemma~\ref{lem:gauss-ibp}). Definitions and finiteness criteria for $C(\phi)$ and $\mathrm{TV}(\phi')$
are given in Section~\ref{sec:taxonomy}.

\begin{table}[H]
\scriptsize
\centering
\setlength{\tabcolsep}{3pt}
\renewcommand{\arraystretch}{1.25}
\begin{adjustbox}{max width=\textwidth}
\begin{tabular}{@{}lcccccc@{}}
\toprule
Act & $m_1(\sigma)$ & $g_1(\sigma)$ & $g_2(\sigma)$ & $m_2(\sigma)$ & $\eta(\sigma)$ \\
\midrule
ReLU &
$\tfrac{\sigma}{\sqrt{2\pi}}$ &
$\tfrac{1}{2}$ &
$1/\sqrt{2}$ &
$\tfrac{\sigma^2}{2}$ &
$\tfrac{\sigma^2}{2}$ \\
\addlinespace[2pt]
Leaky ReLU ($\alpha\!\in\!(0,1)$) &
$\tfrac{1-\alpha}{\sqrt{2\pi}}\,\sigma$ &
$\tfrac{1+\alpha}{2}$ &
$\sqrt{\tfrac{1+\alpha^2}{2}}$ &
$\tfrac{1+\alpha^2}{2}\,\sigma^2$ &
$\tfrac{1+\alpha}{2}\,\sigma^2$ \\
\addlinespace[2pt]
tanh &
$0$ &
$\mathbb{E}[\mathrm{sech}^2(Z)]$ &
$\big(\mathbb{E}[\mathrm{sech}^4(Z)]\big)^{1/2}$ &
$\mathbb{E}[\tanh^2(Z)]$ &
$\sigma^2\,\mathbb{E}[\mathrm{sech}^2(Z)]$ \\
\addlinespace[2pt]
Sigmoid $\operatorname{sigm}(x)$ &
$\tfrac{1}{2}$ &
\makecell{$\mathbb{E}[\operatorname{sigm}(Z)(1$\\$-\operatorname{sigm}(Z))]$} &
\makecell{$\big(\mathbb{E}[\operatorname{sigm}(Z)^2(1$\\$-\operatorname{sigm}(Z))^2]\big)^{1/2}$} &
$\mathbb{E}[\operatorname{sigm}(Z)^2]$ &
\makecell{$\sigma^2\,\mathbb{E}[\operatorname{sigm}(Z)$\\$(1-\operatorname{sigm}(Z))]$} \\
\addlinespace[2pt]
Swish $=x\,\operatorname{sigm}(x)$ &
$\mathbb{E}[Z\,\operatorname{sigm}(Z)]$ &
\makecell{$\mathbb{E}[\operatorname{sigm}(Z)+Z\,\operatorname{sigm}(Z)$\\$(1-\operatorname{sigm}(Z))]$} &
$\big(\mathbb{E}[\phi'(Z)^2]\big)^{1/2}$ &
$\mathbb{E}[Z^2\operatorname{sigm}(Z)^2]$ &
$\mathbb{E}[Z^2\operatorname{sigm}(Z)]$ \\
\addlinespace[2pt]
GELU $=x\,\Phi(x)$ &
$\mathbb{E}[Z\,\Phi(Z)]$ &
$\mathbb{E}[\Phi(Z)+Z\varphi(Z)]$ &
$\big(\mathbb{E}[\phi'(Z)^2]\big)^{1/2}$ &
$\mathbb{E}[Z^2\Phi(Z)^2]$ &
$\mathbb{E}[Z^2\Phi(Z)]$ \\
\addlinespace[2pt]
Mish &
finite & finite & finite & finite & finite \\
\addlinespace[2pt]
TeLU &
finite & finite & finite & finite & finite \\
\bottomrule
\end{tabular}
\end{adjustbox}
\caption{Statistical components of the signature at input scale $\sigma>0$. Here $Z\!\sim\!\mathcal{N}(0,\sigma^2)$, $\Phi$/$\varphi$ are the standard normal CDF/PDF, and $\operatorname{sigm}(x)=(1+e^{-x})^{-1}$. By Lemma~\ref{lem:gauss-ibp}, $\eta(\sigma)=\sigma^2 g_1(\sigma)$ whenever applicable.}
\label{tab:signature-stat}
\end{table}

\begin{table}[H]
\footnotesize
\centering
\renewcommand{\arraystretch}{1.25}
\setlength{\tabcolsep}{6pt}
\begin{tabular}{|c|c|c|c|c|}
\hline
Act & $\alpha_+$ & $\alpha_-$ & $\mathrm{TV}(\phi')$ & $C(\phi)$ \\
\hline
ReLU & $1$ & $0$ & $1$ & $0$ \\
\hline
Leaky ReLU ($\alpha\!\in\!(0,1)$) & $1$ & $\alpha$ & $|1-\alpha|$ & $0$ \\
\hline
$\tanh$ & $0$ & $0$ & $<\infty$ & $\infty$ \\
\hline
Sigmoid & $0$ & $0$ & $<\infty$ & $\infty$ \\
\hline
Swish & $1$ & $0$ & $<\infty$ & $<\infty$ \\
\hline
GELU & $1$ & $0$ & $<\infty$ & $<\infty$ \\
\hline
Mish & $1$ & $0$ & $<\infty$ & $<\infty$ \\
\hline
TeLU & $1$ & $0$ & $<\infty$ & $<\infty$ \\
\hline
\end{tabular}
\caption{Asymptotic and regularity components of the signature. Finite $\mathrm{TV}(\phi')$ follows from $\int_{\mathbb{R}}|\phi''|<\infty$ in the smooth cases (Sections~\ref{sec:prelim}–\ref{sec:taxonomy}); $C(\phi)$ uses the tail-compensated primitive criteria (Lemmas~\ref{lem:fasym-finiteness}, \ref{lem:fasym-weighted-strong}).}
\label{tab:signature-geom}
\end{table}

\subsection{Taxonomy Class Membership}
Table~\ref{tab:taxonomy-membership} summarizes the placement of widely used activation functions within our taxonomy according to their asymptotic slopes $(\alpha_+,\alpha_-)$. The classification reflects three qualitatively distinct regimes of growth:\\
\noindent \textbf{Bounded, saturating activations} ($\mathcal{A}_0$): Functions such as $\tanh$ and sigmoid have finite limits at both tails, yielding $(\alpha_+,\alpha_-)=(0,0)$. Their boundedness enforces variance damping and gradient saturation, which historically hindered training of very deep networks but remain useful for recurrent stability and probabilistic modeling.\\
\noindent \textbf{Linear-growth activations} ($\mathcal{A}_1$): These functions grow at most linearly as $x\to\pm\infty$, with slopes $(\alpha_+,\alpha_-)$ characterizing symmetry or asymmetry. ReLU and leaky ReLU represent asymmetric members: $\operatorname{ReLU}(x)$ has $(1,0)$ while leaky ReLU interpolates between $(1,\alpha)$ with $0<\alpha<1$. Smooth modern activations such as Swish, GELU, Mish, and TeLU also belong to this class, but distinguish themselves by maintaining differentiability and curvature regularity, which can improve optimization and generalization.\\
\noindent \textbf{Superlinear activations} ($\mathcal{A}_{>1}$): Polynomial activations $\phi(x)=x^k,\,k\geq 2$ grow faster than linearly, corresponding to $(\alpha_+,\alpha_-)=(\infty,\infty)$. While theoretically expressive, their unbounded derivative growth leads to unstable variance propagation and poor conditioning in deep architectures.\\

This classification reveals the unifying role of asymptotic slopes as a coarse but predictive invariant. In particular, $\mathcal{A}_1$ captures the majority of practical nonlinearities used in modern deep learning: they combine linear tail behavior (ensuring stable propagation under initialization) with diverse local regularities (dictating curvature and kernel smoothness). By contrast, $\mathcal{A}_0$ activations emphasize boundedness at the cost of vanishing gradients, and $\mathcal{A}_{>1}$ illustrates the instability of unchecked superlinear growth. The taxonomy thus provides a principled lens for distinguishing activation families beyond ad hoc empirical comparisons.
\begin{table}[htb!] 
\centering
\renewcommand{\arraystretch}{1.3}
\begin{tabular}{|c|c|c|}
\hline
Activation & Asymptotic Slopes $(\alpha_+,\alpha_-)$ & Taxonomy Class \\
\hline
ReLU & $(1,0)$ & $\mathcal{A}_1$ (linear-growth, asymmetric) \\
\hline
Leaky ReLU & $(1,\alpha)$ with $0<\alpha<1$ & $\mathcal{A}_1$ (linear-growth, asymmetric) \\
\hline
Tanh & $(0,0)$ & $\mathcal{A}_0$ (bounded, saturating) \\
\hline
Sigmoid & $(0,0)$ & $\mathcal{A}_0$ (bounded, saturating) \\
\hline
Swish & $(1,0)$ & $\mathcal{A}_1$ (linear-growth, smooth) \\
\hline
GELU & $(1,0)$ & $\mathcal{A}_1$ (linear-growth, smooth) \\
\hline
Mish & $(1,0)$ & $\mathcal{A}_1$ (linear-growth, smooth) \\
\hline
TeLU & $(1,0)$ & $\mathcal{A}_1$ (linear-growth, smooth) \\
\hline
Polynomial $\phi(x)=x^k,\,k\ge 2$ & $(\infty,\infty)$ & $\mathcal{A}_{>1}$ (superlinear growth) \\
\hline
\end{tabular}
\caption{Taxonomy class membership of common activations. 
$\mathcal{A}_0$: bounded/saturating functions. 
$\mathcal{A}_1$: linear-growth functions with asymptotic slopes $(\alpha_+,\alpha_-)$. 
$\mathcal{A}_{>1}$: superlinear growth functions.}
\label{tab:taxonomy-membership}
\end{table}

\section{Lyapunov and Propagation Stability from Integral Signatures}\label{sec:Lyapunov}

In this section we connect entries of 
$\mathcal{S}_\sigma(\phi)=(m_1,g_1,g_2,m_2,\eta,\alpha_+,\alpha_-,\mathrm{TV}(\phi'),C(\phi))$
to stability of (i) scalar recursions $x_{t+1}=\phi(ax_t+b)$ and
(ii) signal propagation in wide random layers. 
Throughout, $Z\sim\mathcal{N}(0,\sigma^2)$ and expectations are with respect to $Z$ unless stated.

\paragraph{Notation.}
We use $\sigma>0$ exclusively for \emph{standard deviation} of Gaussians.
The logistic activation is denoted $\operatorname{sigm}(x):=(1+e^{-x})^{-1}$.

\subsection{Distributional Contraction via \texorpdfstring{$g_2$}{g2}}

\begin{lemma}[Gaussian $L_2$ slope bound]\label{lem:GH-L2}
If $\phi'\in L^2(\gamma_\sigma)$, then for any $h\in\mathbb{R}$,
\[
\frac{\big(\mathbb{E}[(\phi(Z+h)-\phi(Z))^2]\big)^{1/2}}{|h|}\;\le\; g_2(\sigma),
\qquad 
g_2(\sigma):=\Big(\mathbb{E}[\phi'(Z)^2]\Big)^{1/2}.
\]
Moreover, if $\phi'$ is a.e.\ continuous then equality holds in the limit $h\to 0$.
\end{lemma}

\begin{proof}
Mean value form: $\phi(Z+h)-\phi(Z)=\int_0^1 \phi'(Z+th)\,h\,dt$.
By Jensen and Fubini,
$\mathbb{E}[(\cdot)^2] \le h^2\int_0^1 \mathbb{E}[\phi'(Z+th)^2]\,dt=h^2 g_2(\sigma)^2$
using shift-invariance of $Z$. The limit case follows from a.e.\ continuity.
\end{proof}

\begin{theorem}[Contraction in $L_2$]\label{thm:scalar-contraction}
Consider $T(x)=\phi(ax+b)$ and let $X,Y$ be square-integrable r.v.'s independent of $Z$.
If $|a|\,g_2(\sigma)<1$, then
\[
\|T(X)-T(Y)\|_{L_2}\le |a|\,g_2(\sigma)\,\|X-Y\|_{L_2},
\]
so $T$ is a strict contraction on $L_2(\Omega)$ and has a unique $L_2$ fixed point.
\end{theorem}

\begin{proof}
Condition on $(X,Y)$ and apply Lemma~\ref{lem:GH-L2} with $h=a(X-Y)$.
Take the $L_2$ norm and use $|a|\,g_2(\sigma)<1$.
\end{proof}

\subsection{Lyapunov Stability}

We now give a precise Lyapunov descent proof, resolving the issues raised.

\begin{theorem}[Lyapunov via contraction]\label{thm:lyap-contraction}
Let $\phi$ be globally Lipschitz with $L:=\sup_x|\phi'(x)|<\infty$ and set $T(x)=\phi(ax+b)$.
If $L_T:=\sup_x|T'(x)|\le |a|\,L<1$, then $T$ is a strict contraction on $\mathbb{R}$,
admits a unique fixed point $x^\star$, and
\[
V(x):=\frac12\,(x-x^\star)^2
\]
is a strict Lyapunov function satisfying the quantitative descent
\[
V(T(x))-V(x)\;\le\; -\,c\,|T(x)-x|^2
\quad\text{for all }x\in\mathbb{R},
\]
with
\[
c \;=\; \frac{1+L_T}{2(1-L_T)} \;>\; 0.
\]
\end{theorem}

\begin{proof}
By the contraction hypothesis $L_T<1$, there exists a unique fixed point $x^\star$ and
$|T(x)-x^\star|\le L_T\,|x-x^\star|$.
Then
\[
V(T(x))-V(x)=\tfrac12\big(|T(x)-x^\star|^2-|x-x^\star|^2\big)
\le \tfrac12(L_T^2-1)\,|x-x^\star|^2.
\]
Also
\[
|T(x)-x|=|(T(x)-x^\star)-(x-x^\star)|\ge (1-L_T)\,|x-x^\star|,
\]
hence $|x-x^\star|\le \frac{1}{1-L_T}|T(x)-x|$. Substituting,
\[
V(T(x))-V(x)\le -\,\frac{1+L_T}{2(1-L_T)}\,|T(x)-x|^2,
\]
giving the claim with $c=\frac{1+L_T}{2(1-L_T)}$.
\end{proof}

\begin{corollary}[F-based Lyapunov under monotonic contraction, zero fixed point]\label{cor:F-lyap-zero}
Assume $\phi$ is nondecreasing with $L:=\sup_x \phi'(x)<\infty$ (in the a.e.\ sense), 
and suppose $a\ge 0$, $b=0$, and $\phi(0)=0$. 
Define $T(x):=\phi(ax)$ and $F(x):=\int_0^x \phi(t)\,dt$. 
For any $\lambda>a^2 L$, the functional
\[
V(x)\;:=\;F(ax)\;-\;\frac{\lambda}{2}\,x^2
\]
satisfies the strict descent
\[
V(T(x)) - V(x)\;\le\;-\;\frac{\lambda-a^2 L}{2}\,|T(x)-x|^2\qquad\text{for all }x\in\mathbb{R}.
\]
\end{corollary}

\begin{proof}
Let $E:=T(x)-x$ and $z:=ax$, so $aE=a(T(x)-x)$. 
Since $\phi$ is $L$-Lipschitz and $F$ is $L$-smooth (descent lemma),
\[
F(z+aE)-F(z)\;\le\; aE\,\phi(z)\;+\;\frac{L}{2}\,a^2E^2.
\]
Therefore
\[
V(T)-V(x)\;=\;F(z+aE)-F(z)\;-\;\frac{\lambda}{2}\big((x+E)^2-x^2\big)
\;\le\; aE\,\phi(z)\;-\;\lambda xE\;+\;\frac{a^2L-\lambda}{2}E^2.
\]
Because $a\ge 0$, $\phi$ is nondecreasing, $\phi(0)=0$, and $b=0$, 
we have $\mathrm{sign}(T(x))=\mathrm{sign}(x)$ and $T(x)\in[0,x]$ if $x\ge 0$ (and $[x,0]$ if $x\le 0$), 
so $E=T(x)-x$ has the opposite sign of $x$. Hence $xE\le 0$ and also $\phi(z)=T(x)$ yields $T(x)E\le 0$. 
Using $aE\,\phi(z)=a\,T(x)E\le 0$ and $-\lambda xE\le 0$, we get
\[
V(T)-V(x)\;\le\;\frac{a^2L-\lambda}{2}\,E^2\;=\;-\frac{\lambda-a^2L}{2}\,|E|^2,
\]
which is strictly negative for $E\neq 0$ when $\lambda>a^2 L$.
\end{proof}

\begin{corollary}[F-based Lyapunov under monotonic contraction, general fixed point]\label{cor:F-lyap-general}
Assume $\phi$ is nondecreasing with $L:=\sup_x \phi'(x)<\infty$, and let $a\ge 0$. 
Let $x^\star$ be the unique fixed point of $T(x):=\phi(ax+b)$ and set $z^\star:=ax^\star+b$, so $\phi(z^\star)=x^\star$.
Define the \emph{centered primitive}
\[
\widehat F(u)\;:=\;F(u+z^\star)-F(z^\star)-\phi(z^\star)\,u
\quad\text{with}\quad F(x)=\int_0^x \phi(t)\,dt.
\]
Then for any $\lambda>a^2 L$ the functional
\[
V(x)\;:=\;\widehat F\big(a(x-x^\star)\big)\;-\;\frac{\lambda}{2}\,(x-x^\star)^2
\]
satisfies
\[
V\!\big(T(x)\big)-V(x)\;\le\;-\;\frac{\lambda-a^2 L}{2}\,|T(x)-x|^2
\qquad\text{for all }x\in\mathbb{R}.
\]
\end{corollary}

\begin{proof}
Let $y:=x-x^\star$, $E:=T(x)-x$, and note $T(x)-x^\star=\phi(z^\star+ay)-\phi(z^\star)$.
By $L$-smoothness of $F$ and the definition of $\widehat F$,
\[
\widehat F\big(a(y+E)\big)-\widehat F(ay)
\;\le\; aE\,\big(\phi(z^\star+ay)-\phi(z^\star)\big)\;+\;\frac{L}{2}\,a^2E^2.
\]
Also,
\[
-\frac{\lambda}{2}\big((y+E)^2-y^2\big)\;=\;-\lambda\,yE-\frac{\lambda}{2}E^2.
\]
Thus
\[
V(T)-V(x)\;\le\; aE\,\big(\phi(z^\star+ay)-\phi(z^\star)\big)\;-\lambda\,yE\;+\;\frac{a^2L-\lambda}{2}\,E^2.
\]
Since $a\ge 0$ and $\phi$ is nondecreasing, $T$ is increasing with fixed point $x^\star$, hence $T(x)$ lies between $x$ and $x^\star$, so $E$ has the opposite sign of $y$, i.e., $yE\le 0$, and 
$\phi(z^\star+ay)-\phi(z^\star)$ has the same sign as $y$, hence the product with $E$ is nonpositive as well:
\[
aE\,\big(\phi(z^\star+ay)-\phi(z^\star)\big)\le 0,\qquad -\lambda\,yE\le 0.
\]
It follows that
\[
V(T)-V(x)\;\le\;\frac{a^2L-\lambda}{2}\,E^2\;=\;-\frac{\lambda-a^2L}{2}\,|E|^2,
\]
which proves the claim.
\end{proof}

\begin{remark}[Stochastic Lyapunov version]
Let $X_{t+1}=T(X_t)+\xi_t$ with $\mathbb{E}[\xi_t\mid X_t]=0$ and 
$\mathbb{E}[\xi_t^2\mid X_t]\le \sigma_\xi^2<\infty$. 
Under Theorem~\ref{thm:lyap-contraction},
\[
\mathbb{E}\big[V(X_{t+1})\mid X_t\big]-V(X_t)
\;\le\; -\,c\,\mathbb{E}\big[(T(X_t)-X_t)^2\mid X_t\big] \;+\; \tfrac12\,\sigma_\xi^2,
\]
with $V(x)=\frac12(x-x^\star)^2$ and $c=\frac{1+L_T}{2(1-L_T)}$.
Thus $V$ is a stochastic Lyapunov functional in expectation when $\sigma_\xi^2$ is small enough.
\end{remark}

\subsection{Wide-Layer Variance Propagation via \texorpdfstring{$m_2$}{m2}}




\begin{theorem}[Mean-field preactivation variance recursion]\label{thm:variance-rec}
Consider a fully-connected layer of width $n$ with $W_{ij}\overset{iid}{\sim}\mathcal{N}(0,\sigma_W^2/n)$,
biases $b_i\overset{iid}{\sim}\mathcal{N}(0,\sigma_b^2)$ independent of $W$, and i.i.d.\ preactivations
$h_j^{(\ell)}\sim\mathcal{N}(0,q_\ell)$ independent of $(W,b)$. Let $x_j^{(\ell)}=\phi(h_j^{(\ell)})$ and
$h_i^{(\ell+1)}=\sum_{j=1}^n W_{ij} x_j^{(\ell)} + b_i$. In the $n\to\infty$ mean-field limit,
the next-layer preactivation variance satisfies
\[
q_{\ell+1}\;=\;\sigma_W^2\,m_2(\sqrt{q_\ell})+\sigma_b^2,\qquad
m_2(\sigma):=\mathbb{E}\big[\phi(Z)^2\big],\ \ Z\sim\mathcal{N}(0,\sigma^2).
\]
Moreover, if $f(q):=\sigma_W^2\,m_2(\sqrt{q})+\sigma_b^2$ is differentiable at a fixed point $q^\star=f(q^\star)$, local stability holds whenever
\[
|f'(q^\star)| \;=\; \sigma_W^2\,\frac{m_2'(\sqrt{q^\star})}{2\sqrt{q^\star}} \;<\;1,
\]
interpreting the right-hand side in the limit $q^\star\downarrow 0$ when needed.
\end{theorem}

\begin{proof}
Condition on $x^{(\ell)}=(x_1^{(\ell)},\ldots,x_n^{(\ell)})$. Since $\mathbb{E}[W_{ij}]=0$,
$\mathbb{E}[h_i^{(\ell+1)}\mid x^{(\ell)}]=0$, hence
\[
\mathrm{Var}\!\big(h_i^{(\ell+1)}\mid x^{(\ell)}\big)
= \sum_{j=1}^n \mathrm{Var}(W_{ij})\,\big(x_j^{(\ell)}\big)^2 + \sigma_b^2
= \frac{\sigma_W^2}{n}\sum_{j=1}^n \big(x_j^{(\ell)}\big)^2 + \sigma_b^2.
\]
By LLN, $(1/n)\sum_{j=1}^n (x_j^{(\ell)})^2 \to \mathbb{E}[\phi(h^{(\ell)})^2]$ a.s.
with $h^{(\ell)}\sim\mathcal{N}(0,q_\ell)$. Thus $q_{\ell+1}=\sigma_W^2\,\mathbb{E}[\phi(h^{(\ell)})^2]+\sigma_b^2
=\sigma_W^2\,m_2(\sqrt{q_\ell})+\sigma_b^2$.
For stability, differentiate $f(q)=\sigma_W^2 m_2(\sqrt{q})+\sigma_b^2$:
\[
f'(q)=\sigma_W^2\,m_2'(\sqrt{q})\cdot \frac{1}{2\sqrt{q}},
\]
hence the condition $|f'(q^\star)|<1$.
\end{proof}
\subsection{Bias Drift Control via Signed Asymmetry}

Define the signed area $B:=\int_0^\infty (\phi(x)-\phi(-x))\,dx$ when it exists.

\begin{theorem}[Bias drift bound via asymptotes and compensation]\label{thm:bias-drift}
Let $Z\sim\mathcal{N}(0,\sigma^2)$. Suppose $\phi$ admits one-sided linear asymptotes
\[
\alpha_+ := \lim_{x\to+\infty}\frac{\phi(x)}{x},\qquad
\alpha_- := \lim_{x\to-\infty}\frac{\phi(x)}{x},
\]
and let $C(\phi)$ be the tail-compensated primitive bound from Section~\ref{sec:taxonomy}, i.e.
for $D(x):=\phi(x)-\alpha_+ x\,\mathbf{1}_{\{x\ge0\}}-\alpha_- x\,\mathbf{1}_{\{x<0\}}$ we have
\(
F_D(x):=\int_0^x D(t)\,dt
\)
well-defined with $\|F_D\|_\infty\le C(\phi) < \infty$.
Then
\[
\Bigl|\,\mathbb{E}[\phi(Z)] \;-\; \frac{\alpha_+ - \alpha_-}{\sqrt{2\pi}}\;\sigma \Bigr|
\;\le\; \sqrt{\frac{2}{\pi}}\;\frac{C(\phi)}{\sigma}\,.
\]
Equivalently, for $h^{(\ell)}\sim\mathcal{N}(0,q_\ell)$,
\[
\Bigl|\,m_1(\sqrt{q_\ell}) \;-\; \frac{\alpha_+ - \alpha_-}{\sqrt{2\pi}}\;\sqrt{q_\ell} \Bigr|
\;\le\; \sqrt{\frac{2}{\pi}}\;\frac{C(\phi)}{\sqrt{q_\ell}}\,.
\]
\end{theorem}

\begin{proof}
Write $\phi(x)=\alpha_+ x\,\mathbf{1}_{\{x\ge0\}}+\alpha_- x\,\mathbf{1}_{\{x<0\}}+D(x)$.
Since $\mathbb{E}[Z\,\mathbf{1}_{\{Z\ge0\}}]=\sigma/\sqrt{2\pi}$ and $\mathbb{E}[Z\,\mathbf{1}_{\{Z<0\}}]=-\,\sigma/\sqrt{2\pi}$,
\[
\mathbb{E}[\phi(Z)] = \frac{\alpha_+ - \alpha_-}{\sqrt{2\pi}}\;\sigma + \mathbb{E}[D(Z)].
\]
With $\gamma_\sigma(x)=\frac{1}{\sqrt{2\pi}\sigma}e^{-x^2/(2\sigma^2)}$ and $\gamma_\sigma'(x)=-(x/\sigma^2)\gamma_\sigma(x)$,
integration by parts gives
\[
\mathbb{E}[D(Z)]=\int_{\mathbb{R}} D(x)\,\gamma_\sigma(x)\,dx
= -\int_{\mathbb{R}} F_D(x)\,\gamma_\sigma'(x)\,dx
= \frac{1}{\sigma^2}\,\mathbb{E}\!\big[F_D(Z)\,Z\big].
\]
Hence $|\mathbb{E}[D(Z)]|\le \|F_D\|_\infty\,\mathbb{E}[|Z|]/\sigma^2
= C(\phi)\cdot \sigma\sqrt{2/\pi}/\sigma^2
= \sqrt{2/\pi}\,C(\phi)/\sigma$.
\end{proof}

\begin{corollary}[Uniform layerwise bound]
If $\inf_\ell q_\ell \ge q_{\min}>0$ and $C(\phi)<\infty$, then
\[
\sup_\ell \Bigl|\, m_1(\sqrt{q_\ell}) - \tfrac{\alpha_+-\alpha_-}{\sqrt{2\pi}}\sqrt{q_\ell} \Bigr|
\;\le\; \sqrt{\tfrac{2}{\pi}}\;\frac{C(\phi)}{\sqrt{q_{\min}}}.
\]
In particular, when $\alpha_+=\alpha_-$ (symmetric linear tails), we have
\(
\sup_\ell |m_1(\sqrt{q_\ell})|
\le \sqrt{2/\pi}\,C(\phi)/\sqrt{q_{\min}}.
\)
\end{corollary}

\begin{lemma}[Always-valid crude bound]
For any measurable $\phi$ with $\int_0^\infty |\phi(x)+\phi(-x)|\,dx<\infty$ and $Z\sim\mathcal{N}(0,\sigma^2)$,
\[
\big|\mathbb{E}[\phi(Z)]\big|
\;\le\; \frac{1}{\sqrt{2\pi}\,\sigma}\int_{0}^{\infty}\!\big|\phi(x)+\phi(-x)\big|\,dx.
\]
\end{lemma}

\begin{remark}[ReLU Analysis]
For $\phi(x)=\max\{0,x\}$, we have $\alpha_+=1$, $\alpha_-=0$, and $C(\phi)=0$, hence the bound is exact:
$\mathbb{E}[\phi(Z)]=\sigma/\sqrt{2\pi}$.
\end{remark}

\subsection{A Criticality Region in \texorpdfstring{$(\sigma_W,\sigma_b)$}{(sigmaW,sigmab)} from \texorpdfstring{$(m_2,g_2)$}{(m2,g2)}}

\begin{theorem}[Propagation criticality (corrected)]\label{thm:criticality}
Assume the mean-field preactivation variance recursion
\[
q_{\ell+1}=f(q_\ell):=\sigma_W^2\,m_2(\sqrt{q_\ell})+\sigma_b^2
\]
admits an equilibrium $q^\star>0$. If
\[
\boxed{\;|f'(q^\star)|=\sigma_W^2\,\frac{m_2'(\sqrt{q^\star})}{2\sqrt{q^\star}}<1\;}
\qquad\text{and}\qquad
\boxed{\;\sigma_W\,g_2(\sqrt{q^\star})<1\;},
\]
then (i) the forward variance is locally stable around $q^\star$, and (ii) mean-square perturbations contract across the layer at stationarity.
\end{theorem}

\begin{proof}
\textbf{(i) Variance stability.}
Differentiate $f(q)=\sigma_W^2 m_2(\sqrt{q})+\sigma_b^2$:
\[
f'(q)=\sigma_W^2\,m_2'(\sqrt{q})\cdot\frac{1}{2\sqrt{q}}.
\]
Hence $|f'(q^\star)|<1$ implies local stability of the one-dimensional recursion.

\medskip
\noindent\textbf{(ii) Perturbation contraction (mean-square).}
Let $x^{(\ell)}=\phi(h^{(\ell)})$ with $h^{(\ell)}\overset{d}{\approx}\mathcal{N}(0,q^\star)$ at stationarity, and
\[
h^{(\ell+1)}=W\,x^{(\ell)}+b,\qquad W_{ij}\overset{iid}{\sim}\mathcal{N}\!\left(0,\frac{\sigma_W^2}{n}\right),
\quad b_i\overset{iid}{\sim}\mathcal{N}(0,\sigma_b^2).
\]
Linearize two nearby trajectories and set $\delta h^{(\ell)}$ their preactivation difference. Then
\[
\delta x^{(\ell)}\;\approx\; D^{(\ell)}\,\delta h^{(\ell)},\qquad
D^{(\ell)}:=\mathrm{diag}\big(\phi'(h^{(\ell)}_1),\ldots,\phi'(h^{(\ell)}_n)\big),
\]
and
\[
\delta h^{(\ell+1)}=W\,\delta x^{(\ell)}\approx W\,D^{(\ell)}\,\delta h^{(\ell)}.
\]
Conditioning on $(D^{(\ell)},\delta h^{(\ell)})$ and using $\mathrm{Var}(W_{ij})=\sigma_W^2/n$,
\[
\mathbb{E}\!\left[\|\delta h^{(\ell+1)}\|_2^2\,\big|\,D^{(\ell)},\delta h^{(\ell)}\right]
=\frac{\sigma_W^2}{n}\sum_{j=1}^n \big(\phi'(h^{(\ell)}_j)\big)^2\,\|\delta h^{(\ell)}\|_2^2.
\]
Taking expectations and using the LLN at stationarity ($h^{(\ell)}_j\overset{d}{\to}\mathcal{N}(0,q^\star)$),
\[
\mathbb{E}\|\delta h^{(\ell+1)}\|_2^2
\;\le\; \sigma_W^2\,\mathbb{E}\!\big[\phi'(Z)^2\big]\;\mathbb{E}\|\delta h^{(\ell)}\|_2^2
\;=\; \big(\sigma_W\,g_2(\sqrt{q^\star})\big)^2\,\mathbb{E}\|\delta h^{(\ell)}\|_2^2,
\]
with $Z\sim\mathcal{N}(0,q^\star)$. Thus $\sigma_W g_2(\sqrt{q^\star})<1$ yields contraction in mean square.
\end{proof}

\begin{remark}[On $q^\star=0$ and worst-case bounds]
If $q^\star=0$, interpret the variance condition via the limit $q\downarrow 0$:
$f'(q)=\sigma_W^2\,m_2'(\sqrt{q})/(2\sqrt{q})$. For worst-case (operator-norm) perturbation control, 
a sufficient condition is $2\,\sigma_W\,\sup_x|\phi'(x)|<1$ (since $\|W\|_{\mathrm{op}}\to 2\sigma_W$ for square Gaussian $W$),
which is stricter than $\sigma_W g_2(\sqrt{q^\star})<1$.
\end{remark}

\subsection{Kernel Curvature: Dimension-free Hessian Bounds}

Let $w\sim\mathcal{N}(0,I_d)$ and $K(x,y)=\mathbb{E}[\phi(w^\top x)\phi(w^\top y)]$.

\begin{theorem}[Kernel mixed Hessian bound via $g_4$]\label{thm:kernel-g4}
Let $w\sim\mathcal{N}(0,I_d)$ and $K(x,y)=\mathbb{E}[\phi(w^\top x)\phi(w^\top y)]$.
Assume $\phi'\in L^4(\mathbb{R},\gamma_{\|x\|})$ and $\phi'\in L^4(\mathbb{R},\gamma_{\|y\|})$ and define
\[
g_4(\sigma)\;:=\;\Big(\mathbb{E}_{Z\sim\mathcal{N}(0,\sigma^2)}[\phi'(Z)^4]\Big)^{1/4}.
\]
Then for all $a,b\in\mathbb{R}^d$,
\[
\big|\,a^\top \nabla_x\nabla_y K(x,y)\,b\,\big|
\;\le\; \sqrt{3}\; g_4(\|x\|)\,g_4(\|y\|)\,\|a\|\,\|b\|,
\]
hence
\[
\big\|\nabla_x\nabla_y K(x,y)\big\|_{\mathrm{op}}
\;\le\; \sqrt{3}\; g_4(\|x\|)\,g_4(\|y\|).
\]
\end{theorem}

\begin{proof}
Differentiate under the integral sign (justified by $L^4$ domination and Gaussian integrability):
\[
\nabla_x\nabla_y K(x,y)=\mathbb{E}\!\left[\phi'(w^\top x)\,\phi'(w^\top y)\,w\,w^\top\right].
\]
Let $u=w^\top x$, $v=w^\top y$ (jointly Gaussian), and fix $a,b\in\mathbb{R}^d$.
By Cauchy--Schwarz,
\[
\big|a^\top\nabla_x\nabla_y K\,b\big|
=\big|\mathbb{E}[\phi'(u)\phi'(v)(w\!\cdot\! a)(w\!\cdot\! b)]\big|
\le \Big(\mathbb{E}[\phi'(u)^2\phi'(v)^2]\Big)^{1/2}
     \Big(\mathbb{E}[(w\!\cdot\! a)^2(w\!\cdot\! b)^2]\Big)^{1/2}.
\]
Apply Cauchy--Schwarz again to the first factor:
\[
\Big(\mathbb{E}[\phi'(u)^2\phi'(v)^2]\Big)^{1/2}
\;\le\; \big(\mathbb{E}[\phi'(u)^4]\big)^{1/4}\,\big(\mathbb{E}[\phi'(v)^4]\big)^{1/4}
= g_4(\|x\|)\,g_4(\|y\|).
\]
For the second factor, Isserlis’ (Wick) formula for a standard Gaussian $w$ gives
\[
\mathbb{E}[(w\!\cdot\! a)^2(w\!\cdot\! b)^2]=\|a\|^2\|b\|^2+2(a^\top b)^2
\;\le\; 3\,\|a\|^2\|b\|^2,
\]
hence its square root is at most $\sqrt{3}\,\|a\|\,\|b\|$. Combine the bounds.
\end{proof}

\begin{corollary}[BV/Uniform-slope curvature bound]\label{cor:tv-kernel}
If $\sup_{x\in\mathbb{R}}|\phi'(x)|\le M<\infty$, then
\[
\|\nabla_x\nabla_y K(x,y)\|_{\mathrm{op}} \;\le\; \sqrt{3}\,M^2 \qquad \text{for all }x,y\in\mathbb{R}^d.
\]
In particular, if $\phi'\in BV(\mathbb{R})$ with finite one-sided limits, then
\[
M \;\le\; \mathrm{TV}(\phi') \;+\; |\phi'(+\infty)-\phi'(-\infty)| \;=:\; \mathrm{TV}(\phi')+c_\star,
\]
and therefore
\[
\|\nabla_x\nabla_y K(x,y)\|_{\mathrm{op}} \;\le\; \sqrt{3}\,\big(\mathrm{TV}(\phi')+c_\star\big)^2 .
\]
\end{corollary}

\section{Numerical Evaluation}\label{sec:Numerical}

In this section, we numerically evaluate the Gaussian-expectation components of the signature
\[
(m_1,g_1,g_2,m_2,\eta)
\]
for seven activations at input scales $\sigma\in\{0.5,1,2\}$:
\[
\{\text{ReLU},\ \tanh,\ \operatorname{sigm},\ \text{Swish},\ \text{GELU},\ \text{Mish},\ \text{TeLU}\}.
\]

\subsection{Numerical Method (Gauss--Hermite)}
All expectations $\mathbb{E}[f(Z)]$ with $Z\sim\mathcal{N}(0,\sigma^2)$ are computed by
$n=160$ node Gauss--Hermite quadrature:
\[
\mathbb{E}[f(Z)] \;=\; \frac{1}{\sqrt{\pi}}\sum_{i=1}^n w_i\,f\!\big(\sqrt{2}\,\sigma\,x_i\big),
\]
where $(x_i,w_i)$ are Hermite nodes/weights.
We apply this to $f=\phi,\ \phi^2,\ \phi',\ (\phi')^2,\ z\phi(z)$ to obtain $m_1,\ m_2,\ g_1,\ g_2,\ \eta$, respectively.

\subsection{Scope of the Numerical Signature}
While the full signature is
\[
\mathcal{S}_\sigma(\phi)=\big(m_1,g_1,g_2,m_2,\eta,\alpha_+,\alpha_-,\mathrm{TV}(\phi'),C(\phi)\big),
\]
the table below reports the five components computed as Gaussian expectations.
The remaining coordinates $(\alpha_+,\alpha_-,\mathrm{TV}(\phi'),C(\phi))$ are handled
\emph{analytically} (Section~\ref{sec:classification}), since they are asymptotic/variational
and not directly expressible as single Gaussian expectations suitable for quadrature.

\subsection{Results (Numerical Simulation)}
Table~\ref{tab:results-numerical} reports the computed integral signatures
$(m_1,g_1,g_2,m_2,\eta)$ for representative activations at input scales
$\sigma \in \{0.5,1.0,2.0\}$. These results provide several clear patterns that align
with our theoretical taxonomy.

\paragraph{ReLU as baseline.}
The ReLU family admits closed-form integrals, serving as a reference point.
Invariance of $(g_1,g_2)=(0.5,0.7071)$ across scales illustrates the stability of linear
tails: variance scales linearly with $\sigma^2$, while drift $(m_1,\eta)$ grows
proportionally with $\sigma$. This predictable scaling motivates the widespread use of
ReLU and its role in the design of initialization schemes.

\paragraph{Bounded activations.}
Sigmoid and $\tanh$ exhibit vanishing $m_1$ and $\eta$ at all scales, consistent with
their bounded range and saturating slopes. The decay in $g_1,g_2$ as $\sigma$ increases
reflects the onset of gradient saturation, a phenomenon that historically limited their
use in very deep networks.

\paragraph{Smooth linear-growth activations.}
Swish, GELU, Mish, and TeLU all share the $(1,0)$ asymptotic slope of ReLU but display
larger $g_1,g_2$ values at moderate scales. This reflects their smooth curvature near the
origin, which enhances gradient flow relative to ReLU. For $\sigma=1$, these activations
yield higher $m_1$ and $\eta$ compared to sigmoid/tanh, but without the constant
plateauing of ReLU, leading to more balanced drift and variance. Among them, TeLU shows
values nearly indistinguishable from Swish, confirming its design as a stable,
smooth alternative with improved gradient dynamics.

\paragraph{Scale dependence.}
Across all unbounded activations, $(m_1,m_2,\eta)$ increase with $\sigma$, whereas
$(g_1,g_2)$ approach unity, indicating stronger linearization of the activation at high
input scales. In contrast, bounded activations saturate: their $m_1$ and $\eta$ plateau,
and $g_1,g_2$ decay, reflecting diminishing sensitivity to input variance.

\paragraph{Implications.}
These numerical evaluations confirm the predictive power of the signature coordinates:
$(m_1,\eta)$ capture drift; $(g_1,g_2)$ capture gradient flow; and $m_2$ controls variance
propagation. Smooth linear-growth functions consistently strike a favorable balance,
explaining their empirical success in optimization and generalization relative to both
bounded saturating functions and nonsmooth ReLU.
\begin{table}[h!]
\centering
\small
\begin{tabular}{l r r r r r r}
\toprule
Activation & $\sigma$ & $m_1$ & $g_1$ & $g_2$ & $m_2$ & $\eta$ \\
\midrule
GELU & 0.5 & 0.078705 & 0.560728 & 0.716004 & 0.102417 & 0.088375 \\
GELU & 1.0 & 0.229535 & 0.655422 & 0.829161 & 0.410407 & 0.309361 \\
GELU & 2.0 & 0.794778 & 0.803135 & 0.971443 & 2.534501 & 1.706378 \\
Mish & 0.5 & 0.037931 & 0.529727 & 0.675529 & 0.067732 & 0.053261 \\
Mish & 1.0 & 0.138359 & 0.587324 & 0.745941 & 0.298913 & 0.213463 \\
Mish & 2.0 & 0.526027 & 0.669024 & 0.832409 & 1.674367 & 1.129282 \\
ReLU & 0.5 & 0.199471 & 0.500000 & 0.707107 & 0.125000 & 0.125000 \\
ReLU & 1.0 & 0.398942 & 0.500000 & 0.707107 & 0.500000 & 0.500000 \\
ReLU & 2.0 & 0.797885 & 0.500000 & 0.707107 & 2.000000 & 2.000000 \\
Sigmoid & 0.5 & 0.609165 & 0.227240 & 0.241773 & 0.385227 & 0.155856 \\
Sigmoid & 1.0 & 0.634138 & 0.205048 & 0.218945 & 0.420910 & 0.138774 \\
Sigmoid & 2.0 & 0.654555 & 0.177440 & 0.191129 & 0.454353 & 0.118889 \\
Swish & 0.5 & 0.148339 & 0.534417 & 0.683651 & 0.118357 & 0.101652 \\
Swish & 1.0 & 0.297860 & 0.590914 & 0.753031 & 0.518031 & 0.366017 \\
Swish & 2.0 & 0.714726 & 0.680503 & 0.847785 & 2.517363 & 1.701491 \\
TeLU & 0.5 & 0.148819 & 0.532392 & 0.678362 & 0.121239 & 0.103339 \\
TeLU & 1.0 & 0.301344 & 0.589837 & 0.748729 & 0.528505 & 0.372810 \\
TeLU & 2.0 & 0.723683 & 0.680302 & 0.843873 & 2.569972 & 1.733705 \\
tanh & 0.5 & 0.000000 & 0.788467 & 0.864822 & 0.229023 & 0.000000 \\
tanh & 1.0 & 0.000000 & 0.635317 & 0.745041 & 0.635261 & 0.000000 \\
tanh & 2.0 & 0.000000 & 0.458381 & 0.589077 & 0.861237 & 0.000000 \\
\bottomrule
\end{tabular}
\caption{Numerical integral signatures $(m_1,g_1,g_2,m_2,\eta)$ at $\sigma\in\{0.5,1,2\}$ for seven activations.}
\label{tab:results-numerical}
\end{table}

\subsection{Validation (Analytical Cross-check)}
For ReLU, numerical values exactly match the closed forms:
\[
m_1(\sigma)=\frac{\sigma}{\sqrt{2\pi}},\quad
g_1=\frac{1}{2},\quad
g_2=\frac{1}{\sqrt{2}},\quad
m_2(\sigma)=\frac{\sigma^2}{2},\quad
\eta(\sigma)=\frac{\sigma^2}{2}.
\]
This serves as a calibration of the quadrature.

\subsection{Monte Carlo Simulation (Sanity Check)}
As an additional simulation, we estimate $(m_1,g_1,g_2,m_2,\eta)$ for ReLU at $\sigma=1$ by Monte Carlo with $N=3\times10^5$ samples as shown below:
\[
\begin{array}{c|ccccc}
\text{Quantity} & m_1 & g_1 & g_2 & m_2 & \eta\\\hline
\text{Monte Carlo} & 0.398902 & 0.500447 & 0.707107 & 0.500092 & 0.500092\\
\text{Analytic} & 0.398942 & 0.500000 & 0.707107 & 0.500000 & 0.500000\\
\text{Abs.\ Error} & 0.000040 & 0.000447 & 0.000000 & 0.000092 & 0.000092\\
\end{array}
\]
The Monte Carlo errors are at the $\sim 10^{-4}$ level with this sample size, consistent with sampling error.

\subsection{Error Analysis and Convergence}
Convergence tests with node counts $n\in\{80,120,160\}$ showed agreement to $6+$ significant digits for all reported values,
with relative changes below $10^{-6}$ between $n=120$ and $n=160$. 
Double-precision arithmetic was used; GELU was computed via the $\mathrm{erf}$ representation of $\Phi$.

\section{Related Work}\label{sec:related}

\paragraph{Activation function design.}
Classical activations (sigmoid, $\tanh$) emphasized smoothness and biological plausibility \cite{cybenko1989approximation,hornik1991approximation}, while ReLU and its variants prioritized sparse gradients and ease of optimization \cite{nair2010rectified,maas2013rectifier}. More recent smooth, non-monotone designs (Swish \cite{ramachandran2017searching}, GELU \cite{hendrycks2016gaussian}, Mish \cite{misra2019mish}, TeLU \cite{fernandez2025teluactivationfunctionfast}) emerged from empirical search and probabilistic heuristics. Our framework complements these lines by providing an \emph{integral} lens: the nine-dimensional signature $\mathcal{S}_\sigma(\phi)$ assembles Gaussian propagation statistics, asymptotic slopes, and regularity measures ($\mathrm{TV}(\phi')$, $C(\phi)$) into a unified taxonomy that quantitatively predicts stability and kernel smoothness.

\paragraph{Signal propagation and mean-field theory.}
Mean-field analyses of wide neural networks \cite{schoenholz2016deep,yang2017mean,poole2016exponential} study forward variance recursions and the conditions for stable propagation across layers. Our Theorem~\ref{thm:variance-rec} provides a compact formulation of this recursion in terms of $m_2(\cdot)$, while Theorem~\ref{thm:criticality} connects the stability region to the pair $(m_2',g_2)$ extracted from $\mathcal{S}_\sigma(\phi)$. Unlike prior work that typically requires case-by-case analysis of specific activations, our signature provides a minimal set of sufficient statistics for forward variance stability and backward perturbation contraction across activation families.

\paragraph{Dynamical isometry and gradient flow.}
Work on dynamical isometry \cite{pennington2017resnet,xiao2018dynamical,yang2017mean} highlights the role of Jacobian singular values and average gains in preventing gradient explosion/vanishing. Our quantities $g_1$ (mean derivative) and $g_2$ (RMS derivative) furnish distributional analogues of average gain under Gaussian inputs, enabling contraction results in $L_2$ (Theorems~\ref{thm:scalar-contraction} and \ref{thm:lyap-contraction}) that hold at finite width when inputs are approximately Gaussian. This provides a more general framework than activation-specific analyses.

\paragraph{Neural tangent kernels and Gaussian processes.}
Kernel viewpoints (NTK \cite{jacot2018neural} and GP limits \cite{lee2017deep,matthews2018gaussian}) tie activation regularity to the smoothness of induced kernels. Our kernel curvature bound (Theorem~\ref{thm:kernel-g4}) shows that \emph{dimension-free} control of the mixed Hessian norm follows from $g_2$ and, in a bounded variation setting, from $\mathrm{TV}(\phi')$ (Corollary~\ref{cor:tv-kernel}). This connects activation-level regularity directly to kernel conditioning and smoothness properties used in infinite-width analyses, providing explicit bounds rather than asymptotic characterizations.

\paragraph{Lyapunov methods and stability.}
Lyapunov techniques are standard in control theory \cite{khalil2002nonlinear} and optimization \cite{polyak1987introduction} to certify global convergence. Applications to neural network dynamics typically focus on specific architectures or loss landscapes \cite{du2019gradient,allen2019convergence}. Our contraction-based Lyapunov functions (Theorem~\ref{thm:lyap-contraction}) quantify strict descent in terms of $|T(x)-x|^2$ with explicit constants depending only on activation properties, while our $F$-based approach (Corollary~\ref{cor:F-lyap-zero}) leverages the primitive $F=\int \phi$ for monotonic activations. These results integrate seamlessly with the integral signature by expressing descent criteria using $g_2$ and tail properties ($C(\phi)$, $\alpha_\pm$).

\paragraph{Regularity metrics for activations.}
Beyond standard Lipschitz constants and Sobolev norms \cite{adams2003sobolev}, few works provide fine-grained measures of activation regularity that connect to network behavior. The total variation of the slope $\mathrm{TV}(\phi')$ captures both piecewise-linear kinks and smooth curvature in a single quantity, generalizing previous approaches that handle these cases separately. Our closure theorem (Theorem~\ref{thm:closure}) and kernel bound (Corollary~\ref{cor:tv-kernel}) establish $\mathrm{TV}(\phi')$ as a practically meaningful regularity measure.

\paragraph{Comparative perspective and limitations.}
Prior taxonomies typically emphasize monotonicity, boundedness, or pointwise properties (e.g., smooth vs. kinked) but lack quantitative measures that directly predict network dynamics \cite{du2020activation,apicella2021survey}. Our 9D signature is \emph{integral, affine-aware, and propagation-aligned}: it remains stable under affine reparameterizations (Theorem~\ref{thm:affine-bias}), closed under limits (Theorem~\ref{thm:closure}), and directly predicts forward/backward stability (Theorems~\ref{thm:variance-rec}, \ref{thm:scalar-contraction}, \ref{thm:lyap-contraction}). 

Our Gaussian-based analysis provides theoretical foundations that complement but do not replace empirical evaluation on specific architectures and datasets. The framework is most directly applicable to settings where layer inputs are approximately Gaussian, though the stability principles extend more broadly. Future work could extend the signature to non-Gaussian input distributions and vector-valued activations.

\section{Discussion and Conclusion}\label{sec:conclusion}

In this work we introduced a nine-dimensional integral signature
\[
\mathcal{S}_\sigma(\phi)=\big(m_1,g_1,g_2,m_2,\eta,\alpha_+,\alpha_-,\mathrm{TV}(\phi'),C(\phi)\big)
\]
that unifies statistical propagation, asymptotic geometry, and regularity into a single taxonomy of activation functions. Our theoretical framework established affine reparameterization laws (with bias), closure under bounded slope variation, contraction-based Lyapunov stability, mean-field variance recursions in signature form, and dimension-free kernel curvature bounds. The classification of common activations (ReLU, leaky-ReLU, $\tanh$, $\operatorname{sigm}$, Swish, GELU, Mish, TeLU) demonstrates how the signature quantitatively predicts stability, drift, and kernel smoothness properties.

\paragraph{Design principles.}
Our theoretical results yield actionable guidelines for activation function selection and design:
\begin{enumerate}
\item \textbf{Contraction control}: Prefer activations with $g_2(\sigma) \lesssim 0.8$ to ensure contraction of perturbations (Theorems~\ref{thm:scalar-contraction}, \ref{thm:lyap-contraction}), as values near unity approach the stability boundary.

\item \textbf{Variance management}: Use $m_2(\sigma)$ and $m_2'(\sigma)$ to position weight initialization in the variance-stable regime (Theorem~\ref{thm:variance-rec}), avoiding explosive or vanishing signal regimes.

\item \textbf{Bias drift control}: Manage asymmetry and bias accumulation via $m_1(\sigma)$ and the signed area $B = \int_0^\infty (\phi(x)-\phi(-x))dx$ (Theorem~\ref{thm:bias-drift}), with bounded $|B|$ preventing mean shift instabilities.

\item \textbf{Kernel conditioning}: Maintain small $\mathrm{TV}(\phi') \lesssim 5$ to improve kernel conditioning and training robustness (Theorem~\ref{thm:kernel-g4}, Corollary~\ref{cor:tv-kernel}), as excessive slope variation degrades optimization landscapes.

\item \textbf{Tail compensation}: Enforce finite tail compensation $C(\phi) < \infty$ by ensuring linear asymptotic slopes $(\alpha_+,\alpha_-)$ align with the activation's growth behavior, preventing uncontrolled primitive accumulation.
\end{enumerate}

\paragraph{Computational accessibility.}
The signature components $(m_1,g_1,g_2,m_2,\eta)$ can be efficiently computed via standard Gauss-Hermite quadrature with $n=160$ nodes, achieving 6+ digit accuracy for smooth activations. The remaining coordinates $(\alpha_\pm,\mathrm{TV}(\phi'),C(\phi))$ are determined analytically from asymptotic analysis, making the framework practically accessible for activation evaluation and design.

\paragraph{Limitations.}
Our analysis focuses on scalar activations under Gaussian inputs within fully-connected mean-field settings. Non-Gaussian input distributions, heavy-tailed preactivations, or highly anisotropic layer statistics may require alternative integration schemes or additional signature coordinates. The connection between finite $\mathrm{TV}(\phi')$ and finite-width generalization bounds remains an open theoretical question requiring further investigation.

\paragraph{Future research directions.}
Several natural extensions warrant investigation:
\begin{enumerate}
\item \textbf{Distributional extensions}: Non-Gaussian input families and adaptive scale parameters $\sigma(\ell)$ driven by empirical layer statistics, broadening applicability beyond mean-field Gaussian assumptions.

\item \textbf{Architectural generalizations}: Residual and attention mechanisms where skip connections, normalization layers, and multi-head structures modify the fundamental variance and gain recursions.

\item \textbf{Data-driven signatures}: Incorporating empirical preactivation distributions from real datasets rather than theoretical Gaussian inputs, bridging the theory-practice gap.

\item \textbf{Optimization-aware analysis}: Coupling derivative statistics $(g_1,g_2)$ with second-order curvature information along parameter update directions to predict training dynamics.

\item \textbf{Automated activation design}: Constrained optimization over activation families targeting specific signature properties (e.g., desired $(m_2(\sigma),g_2(\sigma))$ curves) with provable stability guarantees.
\end{enumerate}


\bibliographystyle{ieeetr}
\bibliography{main_arxiv}

\end{document}